\documentclass[nohyperref]{article}

\usepackage{microtype}
\usepackage{graphicx}
\usepackage{booktabs} 
\usepackage{subcaption}

\usepackage{hyperref}

\usepackage[accepted]{icml2022}

\usepackage{amsmath}
\usepackage{amssymb}
\usepackage{mathtools}
\usepackage{amsthm}

\usepackage[capitalize,noabbrev]{cleveref}

\theoremstyle{plain}
\newtheorem{theorem}{Theorem}[section]
\newtheorem{prop}[theorem]{Proposition}

\theoremstyle{definition}
\newtheorem{defn}[theorem]{Definition}

\theoremstyle{remark}

\newcommand{\iirchi}[2]{\raisebox{0.8\depth}{$#1\chi$}}
\DeclareRobustCommand{\lchi}{{\mathpalette\iirchi\relax}}
\DeclareMathOperator{\diag}{diag}
\newcommand{\abs}[1]{\left\lvert#1\right\rvert}
\newcommand{\norm}[1]{\left\|#1\right\|}
\newcommand{\al}{\alpha}
\newcommand{\la}{\lambda}
\newcommand{\T}{\theta}

\newcommand{\R}{\mathbb{R}}

\usepackage[textsize=tiny]{todonotes}

\icmltitlerunning{COLA: Consistent Learning with Opponent-Learning Awareness}

\begin{document}

\twocolumn[
\icmltitle{COLA: Consistent Learning with Opponent-Learning Awareness}



\icmlsetsymbol{equal}{*}

\begin{icmlauthorlist}
\icmlauthor{Timon Willi}{equal,ox}
\icmlauthor{Alistair Letcher}{equal}
\icmlauthor{Johannes Treutlein}{equal,ut,vi}
\icmlauthor{Jakob Foerster}{ox}
\end{icmlauthorlist}

\icmlaffiliation{ox}{Department of Engineering Science, University of Oxford, United Kingdom}
\icmlaffiliation{vi}{Vector Institute, Toronto, Canada}
\icmlaffiliation{ut}{Department of Computer Science, University of Toronto, Canada}

\icmlcorrespondingauthor{}{timon.willi@eng.ox.ac.uk}

\icmlkeywords{Machine Learning, ICML}

\vskip 0.3in
]



\printAffiliationsAndNotice{\icmlEqualContribution} 

\begin{abstract}
Learning in general-sum games is unstable and frequently leads to socially undesirable (Pareto-dominated) outcomes. To mitigate this, Learning with Opponent-Learning Awareness (LOLA) introduced \emph{opponent shaping} to this setting, by accounting for each agent's influence on their opponents' anticipated learning steps. However, the original LOLA formulation (and follow-up work) is \emph{inconsistent} because LOLA models other agents as \emph{naive learners} rather than LOLA agents.
In previous work, this inconsistency was suggested as a cause of LOLA's failure to preserve stable fixed points (SFPs). First, we formalize consistency and show that higher-order LOLA (HOLA) solves LOLA's inconsistency problem if it converges. Second, we \textit{correct a claim} made in the literature by Schäfer and Anandkumar (2019), proving that Competitive Gradient Descent (CGD) does \emph{not} recover HOLA as a series expansion (and fails to solve the consistency problem).
Third, we propose a new method called Consistent LOLA (COLA), which \emph{learns} update functions that are consistent under mutual opponent shaping. It requires no more than second-order derivatives and learns consistent update functions even when HOLA fails to converge. However, we also prove that even consistent update functions do not preserve SFPs, contradicting the hypothesis that this shortcoming is caused by LOLA's inconsistency.
Finally, in an empirical evaluation on a set of general-sum games, we find that COLA finds prosocial solutions and that it converges under a wider range of learning rates than HOLA and LOLA. We support the latter finding with a theoretical result for a simple game.
\end{abstract}

\section{Introduction}
Much research in deep multi-agent reinforcement learning (MARL) has focused on zero-sum games like Starcraft and Go \citep{silver2017go, vinyals2019alphastar} or fully cooperative settings \citep{oroojlooyjadid2019review}. However, many real-world problems, e.g. self-driving cars, contain both cooperative and competitive elements, and are thus better modeled as general-sum games. One such game is the famous Prisoner's Dilemma \citep{Axelrod84}, in which agents have an individual incentive to defect against their opponent, even though they would prefer the outcome in which both cooperate to the one where both defect. A strategy for the infinitely iterated version of the game (IPD) is tit-for-tat, which starts out cooperating and otherwise mirrors the opponent's last move. It achieves mutual cooperation when chosen by both players and has proven to be successful at IPD tournaments \citep{Axelrod84}. If MARL algorithms are deployed in the real world, it is essential that they are able to cooperate with others and entice others to cooperate with them, using strategies such as tit-for-tat \citep{dafoe2020open}. However, naive gradient descent and other more sophisticated methods~\cite{eg_original,mescheder_numerics_2018,balduzzi_mechanics_2018,mazumdar_finding_2019,schafer_competitive_2020} converge to the mutual defection policy under random initialization \citep{letcher_stable_2019}.
%

An effective paradigm to improve learning in general-sum games is \emph{opponent shaping}, where agents take into account their influence on the anticipated learning step of the other agents. LOLA \citep{foerster_learning_2018} was the first work to make explicit use of opponent shaping and is one of the only general learning methods designed for general-sum games that obtains mutual cooperation with the tit-for-tat strategy in the IPD. 
While LOLA discovers these prosocial equilibria, the original LOLA formulation is inconsistent because LOLA agents assume that their opponent is a \emph{naive learner}. This assumption is clearly violated if two LOLA agents learn together. It has been suggested that this inconsistency is the cause for LOLA's main shortcoming, which is not maintaining the stable fixed points (SFPs) of the underlying game, even in some simple quadratic games (\citeauthor{letcher2018thesis}~\citeyear{letcher2018thesis},~pp.~2,~26; see also \citeauthor{letcher_stable_2019}~\citeyear{letcher_stable_2019}).

\paragraph{Contributions.}
To address LOLA's inconsistency, we first revisit the concept of \emph{higher-order} LOLA (HOLA) \citep{foerster_learning_2018} in Section~\ref{conv-and-consistency-LOLA}. For example, \emph{second-order} LOLA assumes that the opponent is a \emph{first-order} LOLA agent (which in turn assumes the opponent is a naive learner) and so on. Supposing that HOLA converges with increasing order, we define \emph{infinite-order} LOLA (iLOLA) as the limit. Intuitively, two iLOLA agents have a \textit{consistent} view of each other since they accurately account for the learning behavior of the opponent under mutual opponent shaping. Based on this idea, we introduce a formal definition of \textit{consistency} and prove that, if it exists, iLOLA is indeed consistent (Proposition~\ref{ilola-consistent}). 

Second, in Section~\ref{subsec:CGD}, we correct a claim made in previous literature, which would have provided a closed-form solution of the iLOLA update. According to \citet{schafer_competitive_2020}, the Competitive Gradient Descent (CGD) algorithm recovers HOLA as a series expansion. If true, this would imply that CGD coincides with iLOLA, thus solving LOLA's inconsistency problem. We prove that this is untrue: CGD's series expansion does, in general, not recover HOLA, CGD does not correspond to iLOLA, and CGD does not solve the inconsistency problem (Proposition~\ref{cgd-ilola}).

In lieu of a closed-form solution, a naive way of computing the iLOLA update is to iteratively compute higher orders of LOLA until convergence. However, there are two main problems with addressing consistency using a limiting update: the process may diverge and typically requires arbitrarily high derivatives. To address these, in Section~\ref{subsec:COLA}, we propose Consistent LOLA (COLA) as a more robust and efficient alternative. COLA learns a pair of consistent update functions by explicitly minimizing a differentiable measure of consistency inspired by our formal definition. 
We use the representational power of neural networks and gradient based optimization to minimize this loss, resulting in \textit{learned} update functions that are mutually consistent. By reframing the problem as such, we only require up to second-order derivatives. 

In Section~\ref{theoretical-results-cola}, we prove initial results about COLA. First, we show that COLA's solutions are not necessarily unique. Second, despite being consistent, COLA does not recover SFPs, contradicting the prior belief that this shortcoming is caused by inconsistency. Third, to show the benefit of additional consistency, we prove that COLA converges under a wider range of look-ahead rates than LOLA in a simple general-sum game.

Finally, in Sections~\ref{section-experiments} and \ref{results}, we report our experimental setup and results, investigating COLA and HOLA and comparing COLA to LOLA and CGD in a range of games. We experimentally confirm our theoretical result that CGD does not equal iLOLA. 
Moreover, we show that COLA converges under a wider range of look-ahead rates than HOLA and LOLA, and that it is generally able to find socially desirable solutions. It is the only algorithm consistently converging to the fair solution in the Ultimatum game, and while it does not find tit-for-tat in the IPD (unlike LOLA), it does learn policies with near-optimal total payoff. We find that COLA learns consistent update functions even when HOLA diverges with higher order and its updates are similar to iLOLA when HOLA converges. Although COLA solutions are not unique in theory, COLA empirically tends to find similar solutions over different runs.

\section{Related work}

\begin{table*}[hbt!]
\small
\centering
\caption{On the Tandem game: (a) Log of the squared consistency loss, where e.g. HOLA6 is sixth-order higher-LOLA. (b) Cosine similarity between COLA and LOLA, HOLA3, and HOLA6 over different look-ahead rates. The values represent the mean of a 1,000 samples, uniformly sampled from the parameter space $\Theta$. Error bars represent one standard deviation over 10 COLA training runs.}
\subfloat[]{\begin{tabular}{l|l l l l}
\label{tab:beta_hola}$\alpha$ & LOLA  & HOLA3 & HOLA6  & \multicolumn{1}{c}{COLA}            \\ \hline
1.0                                             & 128.0 & 512   & 131072 & 3e-14$\pm$2e-15 \\ \hline
0.5                                             & 12.81 & 14.05 & 12.35  & 2e-14$\pm$5e-15 \\ \hline
0.3                                             & 2.61  & 2.05  & 0.66   & 4e-14$\pm$3e-15 \\ \hline
0.1                                             & 0.08  & 9e-3  & 2e-6   & 6e-14$\pm$9e-15 \\ \hline
0.01                                            & 1e-5  & 2e-8  & 4e-14  & 1e-14$\pm$4e-14 \\ \hline
\end{tabular}}
\quad
\subfloat[]{\begin{tabular}{l|l l l}
\label{tab:tandem_cos}
$\alpha$ & \multicolumn{1}{c}{LOLA} & \multicolumn{1}{c}{HOLA3} & \multicolumn{1}{c}{HOLA6} \\ \hline
1.0      & 0.94$\pm$0.04             & 0.94$\pm$0.04              & 0.94$\pm$0.04             \\ \hline
0.5      & 0.88$\pm$0.12             & 0.88$\pm$0.12              & 0.08$\pm$0.13             \\ \hline
0.3      & 0.92$\pm$0.01             & 0.91$\pm$0.01              & 0.80$\pm$0.01             \\ \hline
0.1      & 0.95$\pm$0.01             & 0.99$\pm$0.01              & 0.99$\pm$0.01             \\ \hline
0.01     & 0.99$\pm$0.01             & 1.00$\pm$0.00                & 1.00$\pm$0.00               \\ \hline
\end{tabular}}
\end{table*}
\label{appendix:related_work}
General-sum learning algorithms have been investigated from different perspectives in the reinforcement learning, game theory, and GAN literature \citep{Schmidhuber:90sab,barto2002hierarchical, goodfellow2014gans, racaniere2017imagination}. Next, we will highlight a few of the approaches to the mutual opponent shaping problem.

Opponent modeling maintains an explicit belief of the opponent, allowing to reason over their strategies and compute optimal responses. Opponent modeling can be divided into different subcategories: There are classification methods, classifying the opponents into pre-defined types \citep{weber2009datamining, synnaeve2011bayesian}, or policy reconstruction methods, where we explicitly predict the actions of the opponent \citep{mealing2017opponent}. Most closely related to opponent shaping is recursive reasoning, where methods model nested beliefs of the opponents \citep{he2016opponent, albrecht2019reasoning, wen2019probabilistic}.

In comparison, COLA assumes that we have access to the ground-truth model of the opponent, e.g., the opponent's payoff function, parameters, and gradients, putting COLA into the framework  of differentiable games \citep{balduzzi_mechanics_2018}. Various methods have been proposed, investigating the local convergence properties to different solution concepts \citep{mescheder_numerics_2018, mazumdar_finding_2019, letcher_stable_2019, schafer_competitive_2020, azizian2020tight,  schafer2020competitive, hutter2020learning}. Most of the work in differentiable games has not focused on opponent shaping or consistency. \citet{mescheder_numerics_2018} and \citet{mazumdar_finding_2019} focus solely on zero-sum games without shaping. To improve upon LOLA, \citet{letcher_stable_2019} suggested Stable Opponent Shaping (SOS), which applies ad-hoc corrections to the LOLA update, leading to theoretically guaranteed convergence to SFPs. However, despite its desirable convergence properties, SOS still does not solve the conceptual issue of inconsistent assumptions about the opponent. CGD \citep{schafer_competitive_2020} addresses the inconsistency issue for zero-sum games but not for general-sum games. The exact difference between CGD, LOLA and our method is addressed in Section~\ref{subsec:CGD}. Model-Free Opponent Shaping (M-FOS) \citep{lu2022mfos} frames opponent shaping as a meta-learning problem requiring only first-order derivatives. M-FOS is consistent in that it does not make any assumption about the learning algorithm of the opponent. However, the work does not investigate consistency specifically.

\section{Background}

\begin{figure*}[htp!]
 \centering
 \begin{subfigure}[]{0.32\linewidth}
     \centering
	\includegraphics[width=\linewidth]{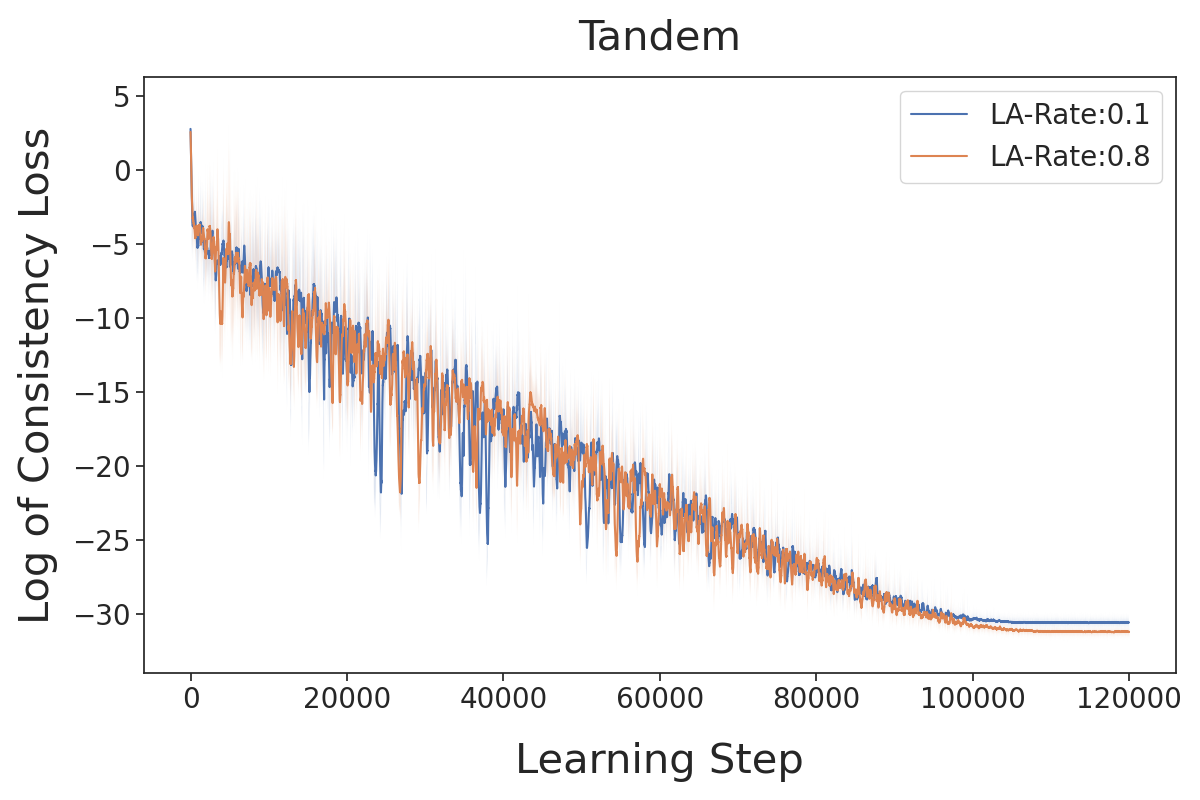}
	\caption{}
	\label{fig:tandem_cons}
 \end{subfigure}
 \begin{subfigure}[]{0.32\linewidth}
    \centering
	\includegraphics[width=\linewidth]{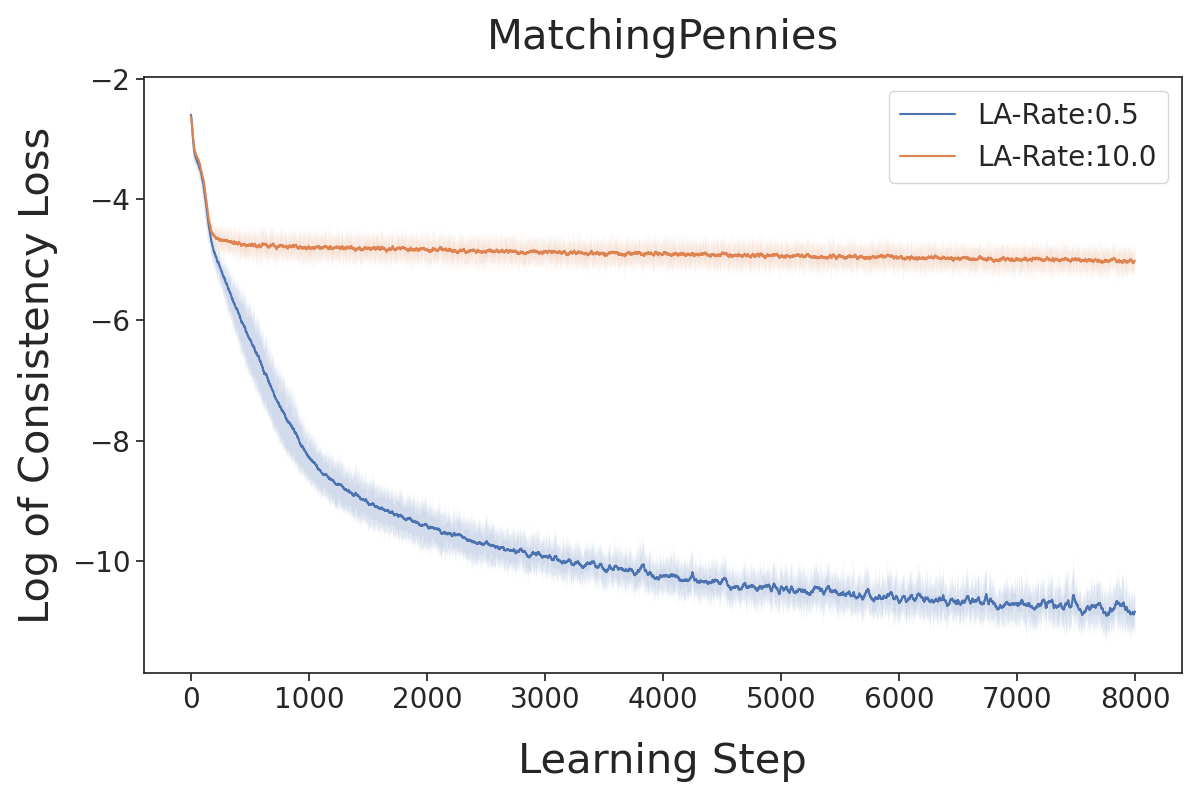}
	\caption{}
	\label{fig:imp_cons}
 \end{subfigure}
 \begin{subfigure}[]{0.32\linewidth}
	\includegraphics[width=\linewidth]{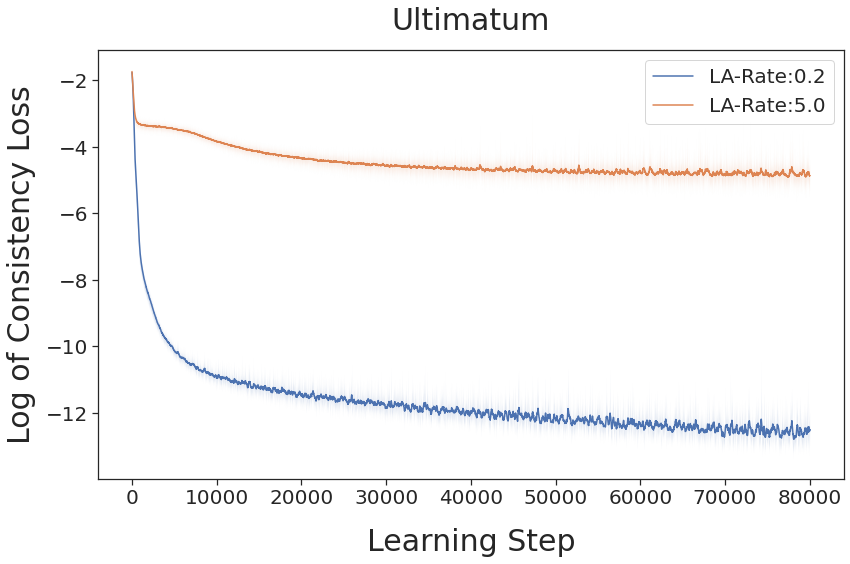}
	\caption{}
	\label{fig:ult_cons}
 \end{subfigure}
\begin{subfigure}[]{0.32\linewidth}
    \centering
  	\includegraphics[width=\linewidth]{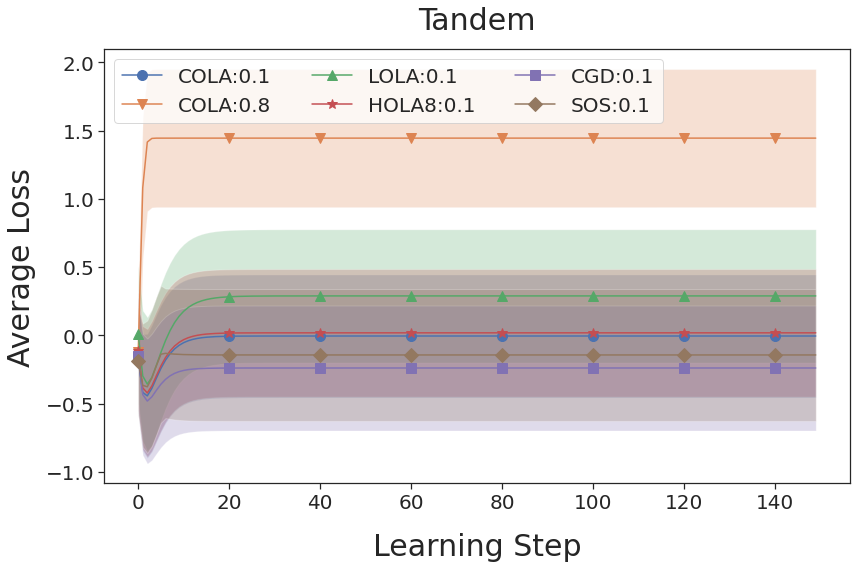}
  	  	\caption{} 
  	\label{fig:tandem_play}
 \end{subfigure}
\begin{subfigure}[]{0.32\linewidth}
    \centering
  	\includegraphics[width=\linewidth]{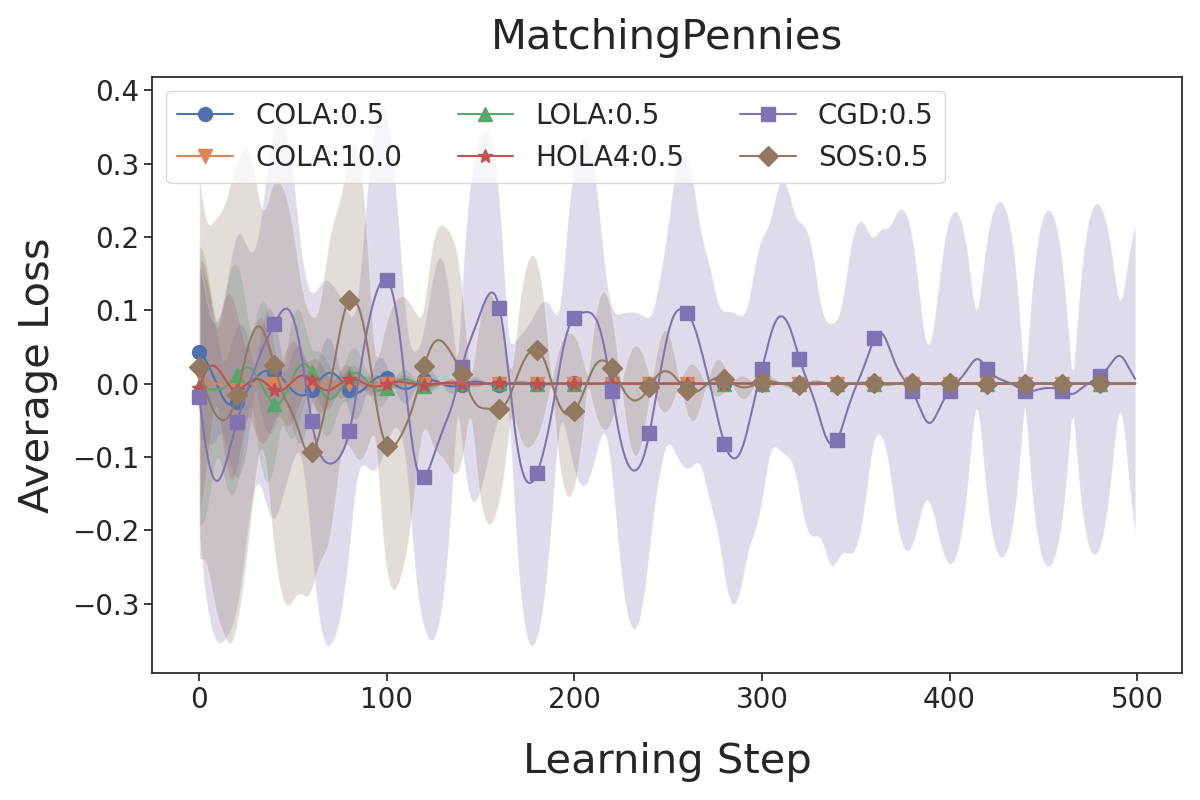}
  	  	\caption{}
  	\label{fig:imp_play}
 \end{subfigure}
\begin{subfigure}[]{0.32\linewidth}
  	\includegraphics[width=\linewidth]{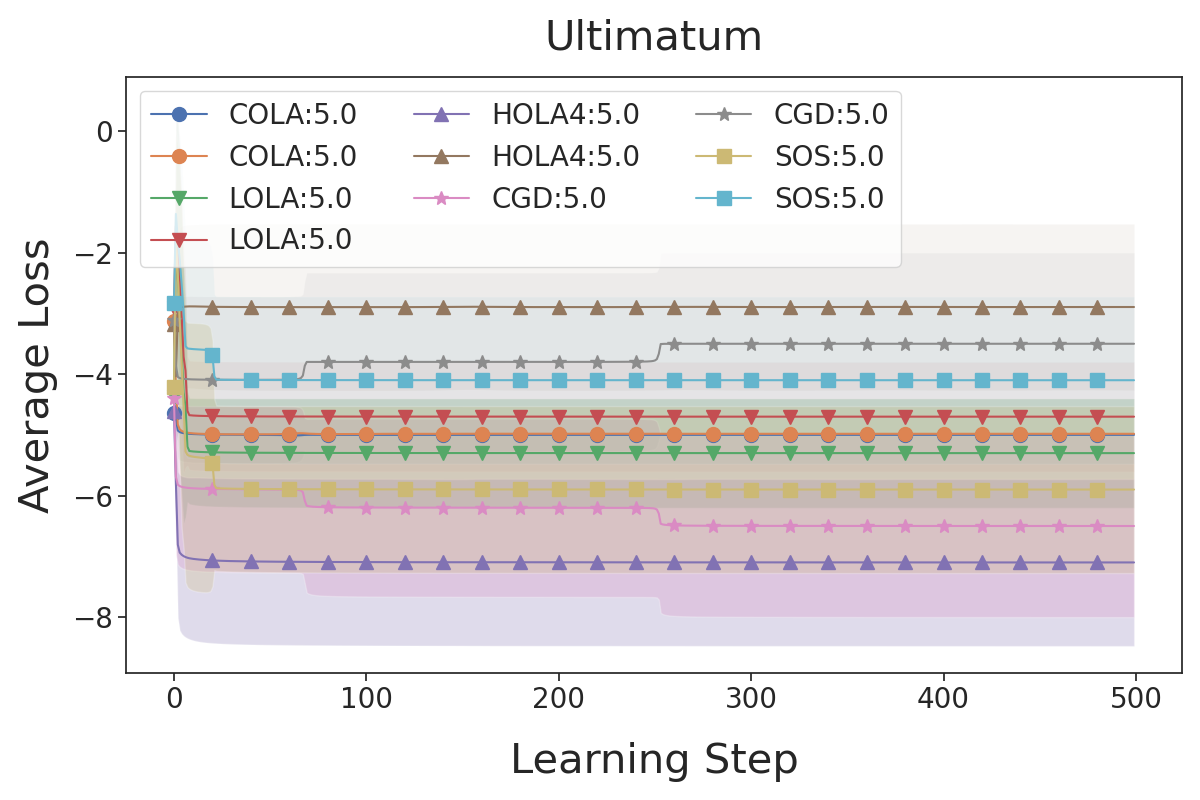}
  	  	\caption{}
  	\label{fig:ult_play}
 \end{subfigure}
 \caption{Subfigures (a), (b) and (c): Log of the consistency loss with standard error of 10 independent training runs over the training of the update functions for the Tandem, MP and Ultimatum games. Subfigures (d), (e) and (f): Learning outcomes for the respective games. E.g., ``COLA:0.1'' is COLA at a look-ahead rate of 0.1. Lines represent payoff means and shaded areas standard deviations over 10 runs.}
 \label{fig:full_game}
\end{figure*}
\subsection{Differentiable games}
The framework of differentiable games has become increasingly popular to model multi-agent learning. Whereas stochastic games are limited to parameters such as action-state probabilities, differentiable games generalize to any real-valued parameter vectors and differentiable loss functions \citep{balduzzi_mechanics_2018}. We restrict our attention to two-player games, as is standard in much of the literature \citep{foerster_learning_2018, foerster_dice_2018, schafer_competitive_2020}.

\begin{defn}[Differentiable games]In a two-player differentiable game, players \(i=1,2\) control parameters $\theta_i\in \mathbb{R}^{d_i}$ to minimize twice continuously differentiable losses $L^{i}: \mathbb{R}^{d_1+d_2} \rightarrow \mathbb{R}$. We adopt the convention to write \(-i\) to denote the opponent of player $i$.
\end{defn}

A fundamental challenge of the multi-loss setting is finding a good solution concept. Whereas in the single loss setting the typical solution concept are local minima, in multi-loss settings there are different sensible solution concepts. Most prominently, there are Nash Equilibria \citep{osborne_gametheory_1994}. However, Nash Equilibria include unstable saddle points that cannot be reasonably found via gradient-based learning algorithms \citep{letcher_stable_2019}. A more suitable concept are stable fixed points (SFPs), which could be considered a differentiable game analogon to local minima in single loss optimization. We will omit a formal definition here for brevity and point the reader to previous work on the topic \citep{letcher_differentiable_2019}.

\subsection{LOLA}
\label{lola-and-sos}
Consider a differentiable game with two players. A LOLA agent $\theta_{1}$ uses its access to the opponent's parameters $\theta_{2}$ to differentiate through a learning step of the opponent. That is, agent 1 reformulates their loss to
$\tilde{L}^1=L^{1}\left(\theta_{1}, \theta_{2}+\widehat{\Delta \theta_{2}}\right)$, where $\widehat{\Delta \theta_{2}}$ represents the assumed learning step of the opponent. In first-order LOLA we assume the opponent to be a naive learner: $\widehat{\Delta \theta_{2}}=-\alpha\nabla_{2} L^{2}$. This assumption makes LOLA inconsistent when the opponent is any other type of learner. Here, $\nabla_{2}$ denotes the gradient with respect to $\theta_{2}$, and $\alpha$ represents the \textit{look-ahead} rate, which is the \textit{assumed} learning rate of the opponent. This rate may differ from the opponent's actual learning rate, but we will only consider equal learning rates and look-ahead rates across opponents for simplicity. In the original paper the loss was approximated using a Taylor expansion,
$\tilde{L}^{1}\approx L^{1}+(\nabla_{2} L^{1})^\top \widehat{\Delta \theta_{2}}$. For agent 1, their first-order \emph{Taylor} LOLA update is then\[\Delta\theta_1= -\alpha\left(\nabla_{1} L^{1}+\nabla_{12} L^{1} \widehat{\Delta \theta_{2}}+\left(\nabla_{1} \widehat{\Delta \theta_{2}}\right)^{\top} \nabla_{2} L^{1}\right).\]
Alternatively, in \emph{exact} LOLA, the derivative is taken directly with respect to the reformulated loss, yielding the update
\[\Delta\theta_1=-\alpha\nabla_1 \left(L^{1}\left(\theta_{1}, \theta_{2}+\widehat{\Delta \theta_{2}}\right)\right).\]

LOLA has had some empirical success, being one of the first general learning methods to discover tit-for-tat in the IPD. However, later work showed that LOLA does not preserve SFPs $\bar{\theta}$, e.g., the rightmost term in the equation for Taylor LOLA can be nonzero at $\bar{\theta}$. In fact, LOLA agents show ``arrogant'' behavior: they assume they can shape the learning of their \emph{naive} opponents without having to adapt to the shaping of the opponent. Prior work hypothesized that this arrogant behavior is due to LOLA's inconsistent formulation and may be the cause for LOLA's failure to preserve SFPs (\citet{letcher2018thesis},~pp.~2,~26; \citet{letcher_stable_2019})

\subsection{CGD}
CGD \citep{schafer_competitive_2020} proposes updates that are themselves Nash Equilibra of a local bilinear approximation of the game. It stands out by its robustness to different look-ahead rates and its ability to find SFPs. However, CGD does not find tit-for-tat on the IPD, instead converging to mutual defection (see Figure \ref{fig:ipd_play}).
CGD's update rule is given by
\begin{equation*}
\left(\begin{array}{c}
\Delta \theta_{1} \\
\Delta \theta_{2}
\end{array}\right) =-\al\left(\begin{array}{cc}
\mathrm{Id} & \alpha \nabla_{1 2} L^{1} \\
\alpha \nabla_{21}L^{2} & \mathrm{Id}
\end{array}\right)^{-1}\left(\begin{array}{c}
\nabla_{1} L^{1} \\
\nabla_{2} L^{2}
\end{array}\right) \,.
\end{equation*}
One can recover different orders of CGD by approximating the inverse matrix via the series expansion ${(\mathrm{Id}-A)^{-1}=\lim _{N \rightarrow \infty} \sum_{k=0}^{N} A^{k}}$ for $\Vert A\Vert<1$. 
For example, at N=1, we recover a version called Linearized CGD (LCGD), defined via \(\Delta \theta_{1}:=-\al\nabla_{1} L^{1}+\alpha^2\nabla_{12} L^{1} \nabla_{2} L^{2}\).

\section{Method and theory}
In this section, we formally define iLOLA and consistency under mutual opponent shaping and show that iLOLA is consistent, thus in principle addressing LOLA's inconsistency problem. We then clear up the relation between CGD and iLOLA, \textit{correcting a false claim} in \citeauthor{schafer_competitive_2020} (\citeyear{schafer_competitive_2020}). Lastly, we introduce COLA as an alternative to iLOLA and present some initial theoretical analysis, including the result that, contrary to prior belief, even consistent update functions do not recover SFPs.

\begin{table*}[hbt!]
\small
\caption{On the MP game: Over different look-ahead rates we compare (a) the consistency losses and (b) the cosine similarity between COLA and LOLA, HOLA2, and HOLA4. The values represent the mean over 1,000 samples, uniformly sampled from the parameter space $\Theta$. The error bars represent one standard deviation and capture the variance over 10 different COLA training runs.}
\centering

\subfloat[]{\begin{tabular}{l|l l l l}
\label{tab:imp_cons}$\alpha$ & \multicolumn{1}{c}{LOLA} & \multicolumn{1}{c}{HOLA2} & \multicolumn{1}{c}{HOLA4} & \multicolumn{1}{c}{COLA} \\ \hline
10       & 0.06                      & 0.70                       & 6.56                       & 2e-3$\pm$3e-4                \\ \hline
5        & 5e-3                   & 0.03                       & 0.15                       & 5e-4$\pm$5e-5                \\ \hline
1.0      & 9e-6                   & 3e-8                    & 4e-9                    & 3e-6$\pm$1e-6                \\ \hline
0.5      & 5e-7                   & 3e-10                   & 5e-12                   & 3e-6$\pm$3e-6                \\ \hline
0.01     & 1e-13                  & 6e-17                   & 5e-17                   & 2e-6$\pm$2e-6                \\ \hline
\end{tabular}}
\quad
\subfloat[]{\begin{tabular}{l|l l l}
\label{tab:imp_sim}$\alpha$ & LOLA      & HOLA2     & HOLA4     \\ \hline
10       & 0.90$\pm$0.01 & 0.84$\pm$0.01 & 0.74$\pm$0.02 \\ \hline
5        & 0.98$\pm$0.01 & 0.97$\pm$0.01 & 0.92$\pm$0.01 \\ \hline
1.0      & 1.00$\pm$0.00 & 1.00$\pm$0.00 & 1.00$\pm$0.00 \\ \hline
0.5      & 1.00$\pm$0.00 & 1.00$\pm$0.00 & 1.00$\pm$0.00 \\ \hline
0.01     & 1.00$\pm$0.00 & 1.00$\pm$0.00 & 1.00$\pm$0.00 \\ \hline
\end{tabular}}
\end{table*}
\subsection{Convergence and consistency of higher-order LOLA}
\label{conv-and-consistency-LOLA}
The original formulation of LOLA is inconsistent when two LOLA agents learn together, because LOLA agents assume their opponent is a naive learner. To address this problem, we define and analyze iLOLA. In this section, we focus on exact LOLA, but we provide a version of our analysis for Taylor LOLA in Appendix~\ref{appendix-Taylor-iLOLA}. HOLA$n$ is defined by the recursive relation
\begin{align*}
h_1^{n+1} &\coloneqq -\al \nabla_1 \left( L^1(\T_1, \T_2 + h_2^{n}) \right) \\
h_2^{n+1} &\coloneqq -\al \nabla_2 \left( L^2(\T_1+h_1^{n}, \T_2) \right)
\end{align*}
with $h_1^{-1} = h_2^{-1} = 0$, omitting arguments $(\T^1, \T^2)$ for convenience. In particular, HOLA\(0\) coincides with simultaneous gradient descent while HOLA\(1\) coincides with LOLA.

\begin{defn}[iLOLA]
If $\textup{HOLA}n$ = $(h_1^n, h_2^n)$ converges pointwise as $n \to \infty$, define
\[ \textup{iLOLA} \coloneqq \lim_{n\to\infty} \begin{pmatrix}
h_1^n \\ h_2^n
\end{pmatrix} \textup{as the limiting update.}\]

\end{defn}

We show in Appendix \ref{appendix:nonconvergence} that HOLA does not always converge, even in simple quadratic games. But, unlike LOLA, iLOLA satisfies a criterion of \textit{consistency} whenever HOLA does converge (under some assumptions), formally defined as follows:
\begin{defn}[Consistency] \label{defn-consistency} Any update functions $f_1\colon \mathbb{R}^{d}\rightarrow \mathbb{R}^{d_1}$ and $f_2\colon \mathbb{R}^d\rightarrow\mathbb{R}^{d_2}$ are \textit{consistent} (under mutual opponent shaping with look-ahead rate \(\alpha\)) if for all \(\theta_1\in\mathbb{R}^{d_1},\theta_2\in\mathbb{R}^{d_2}\), they satisfy
\begin{align}\label{eq:cons1}
    f_1(\theta_1,\theta_2) &= -\alpha\nabla_{1}(L^1(\theta_1, \theta_2 +  f_2(\theta_1,\theta_2)))\\\label{eq:cons2}
    f_2(\theta_1,\theta_2) &= -\alpha\nabla_{2}(L^2(\theta_1 + f_1(\theta_1,\theta_2), \theta_2))
\end{align}
\end{defn}

\begin{prop}\label{ilola-consistent}Let \(\mathrm{HOLA}n=(h_1^n,h_2^n)\) denote both players' exact \(n\)-th order LOLA updates.
Assume that ${\lim_{n\rightarrow\infty}h_i^n(\T)=h_i(\T)}$ and ${\lim_{n\rightarrow\infty}\nabla_{i} h_{-i}^n(\T)=\nabla_ih_{-i}(\T)}$ exist for all $\T\in\mathbb{R}^d$ and ${i \in \{1, 2\}}$. Then iLOLA is consistent under mutual opponent shaping.
\end{prop}
\begin{proof}
In Appendix \ref{appendix:ilolaconv}.
\end{proof}

\subsection{CGD does not recover higher-order LOLA}
\label{subsec:CGD}
\citeauthor{schafer_competitive_2020} (\citeyear{schafer_competitive_2020}) claim that ``LCGD [Linearized CGD] coincides with first order LOLA'' (page 6), and moreover that the higher-order ``series-expansion [of CGD] would recover higher-order LOLA'' (page 4). If this were correct, it would imply that \textit{full} CGD is equal to iLOLA and thus provides a convenient closed-form solution. We prove that this is false in general games: 

\begin{prop}\label{cgd-ilola}
CGD is inconsistent and does not coincide with iLOLA. In particular, Linearized CGD (LCGD) does not coincide with LOLA and the series-expansion of CGD does not recover HOLA (neither exact nor Taylor). Instead, LCGD coincides with LookAhead \citep{Zha}, an algorithm that \textit{lacks} opponent shaping, and the series-expansion of CGD recovers higher-order LookAhead.
\end{prop}
\begin{proof}
In Appendix \ref{appendix:cgd-ilola}. For the negative results, it suffices to construct a single counterexample: we show that LCGD and LOLA differ almost everywhere in the Tandem game (excluding a set of measure zero). We prove by contradiction that the series-expansion of CGD does not recover HOLA. If it did, CGD would equal iLOLA, and by Proposition~\ref{ilola-consistent}, CGD would satisfy the consistency equations. However, this fails almost everywhere in the Tandem game, concluding the contradiction.
\end{proof}

\subsection{COLA}
\label{subsec:COLA}
iLOLA is consistent under mutual opponent shaping. However, HOLA does not always converge and, even when it does, it may be expensive to recursively compute HOLA$n$ for sufficiently high $n$ to achieve convergence.
As an alternative, we propose COLA.

COLA learns consistent update functions and avoids infinite regress by directly solving the equations in Definition~\ref{defn-consistency}. We define the \emph{consistency losses} for learned update functions \(f_1, f_2\) parameterized by \(\phi_1, \phi_2\), obtained for a given \(\T\) as the difference between RHS and LHS in Definition \ref{defn-consistency}:
\begin{align*}
C_1(\phi_1,\phi_2,\theta_1, \theta_2) &= \left\| f_1 + \alpha\nabla_{1}(L^1(\theta_1, \theta_2 +  f_2)) \right\| \\
C_2(\phi_1,\phi_2,\theta_1, \theta_2) &= \left\| f_2 + \alpha\nabla_{2}(L^2(\theta_1 + f_1, \theta_2)) \right\| \,.
\end{align*}

If both losses are $0$ for all \(\T\), then the two update functions defined by \(\phi_1,\phi_2\) are consistent. For this paper, we parameterise $f_1$, $f_2$ as neural networks with parameters $\phi_1$, $\phi_2$ respectively, and numerically minimize the sum of both losses over a region of interest.

The parameter region of interest $\Theta$ depends on the game being played. For games with probabilities as actions, we select an area that captures most of the probability space (e.g. we sample a pair of parameters $\theta_{1}, \theta_{2} \sim [-7, 7]$, since $\sigma(7)\approx1$ where $\sigma$ is the sigmoid function).

We optimize the mean of the sum of consistency losses, 
\begin{multline*}
    C( \phi_1,\phi_2):=    \mathbb{E}_{(\theta_1, \theta_2)\sim\mathcal{U}(\Theta)}\big[ C_1(\phi_1,\phi_2,\theta_1, \theta_2) \\+    C_2(\phi_1,\phi_2,\theta_1, \theta_2) \big],
\end{multline*}
by sampling parameter pairs $(\theta_1, \theta_2)$ uniformly from $\Theta$ and feeding them to the neural networks \(f_1, f_2\), each outputting an agent's parameter update. The weights \(\phi_1,\phi_2\) are then updated by taking a gradient step to minimize $C$. We train the update functions until the loss has converged and use the learned update functions to train a pair of agent policies in the given game.

\begin{figure*}[hbt!]
\centering
\begin{subfigure}{.33\textwidth}
  \centering
  \includegraphics[width=.99\linewidth]{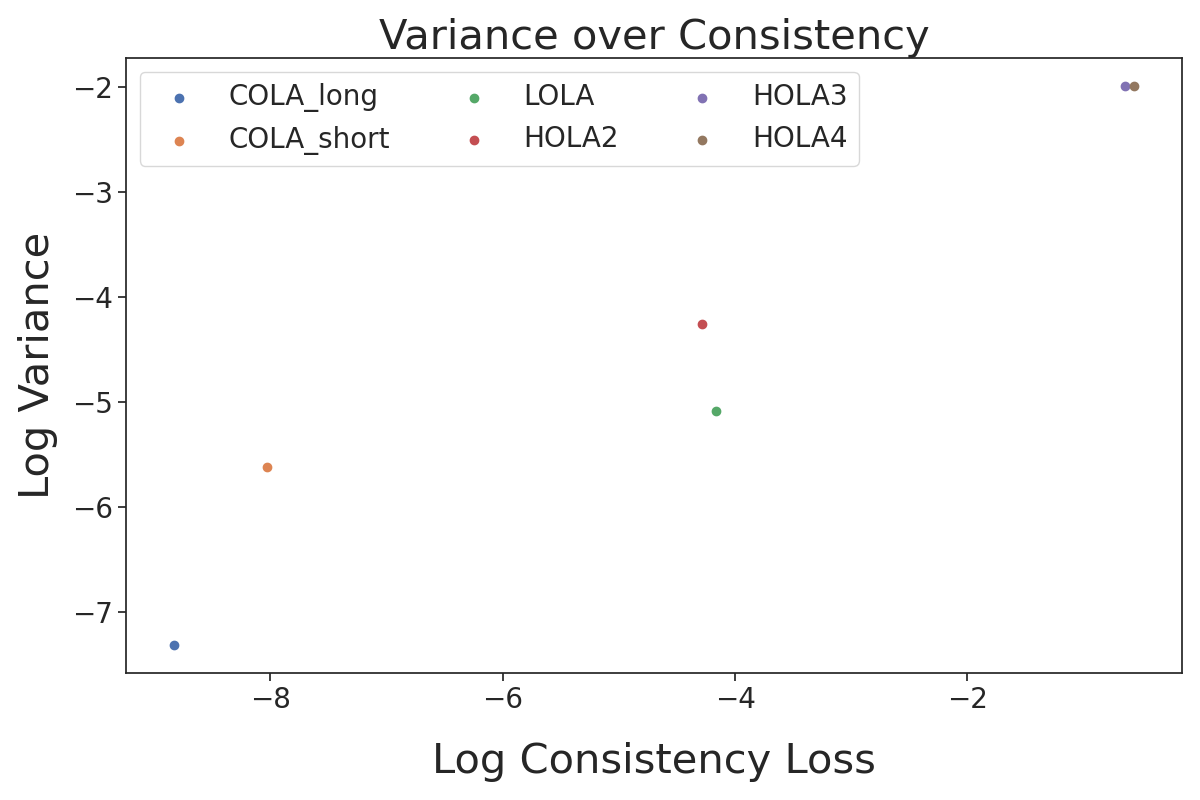}
  \caption{}
  \label{fig:var_cons_high}
\end{subfigure}
\begin{subfigure}{.33\textwidth}
  \centering
  \includegraphics[width=.99\linewidth]{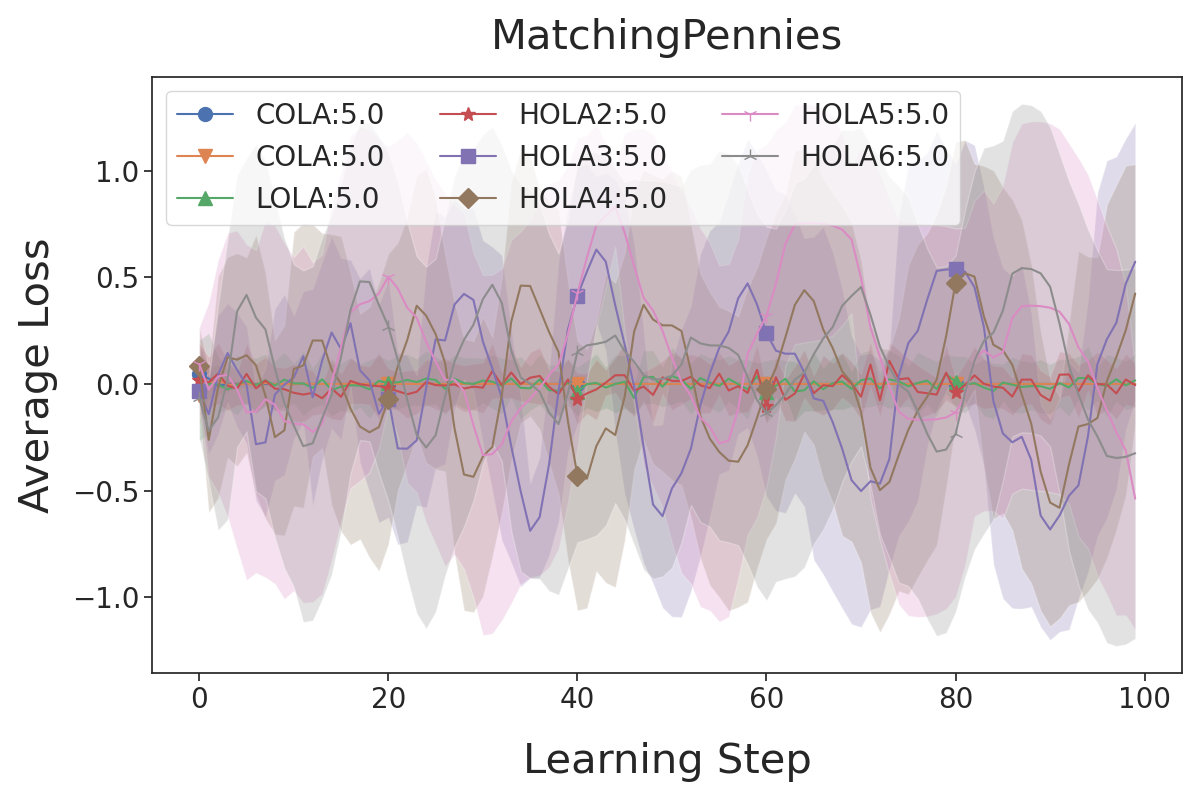}
  \caption{}
  \label{fig:mp_div}
\end{subfigure}
\begin{subfigure}[]{0.33\linewidth}
  	\includegraphics[width=\linewidth]{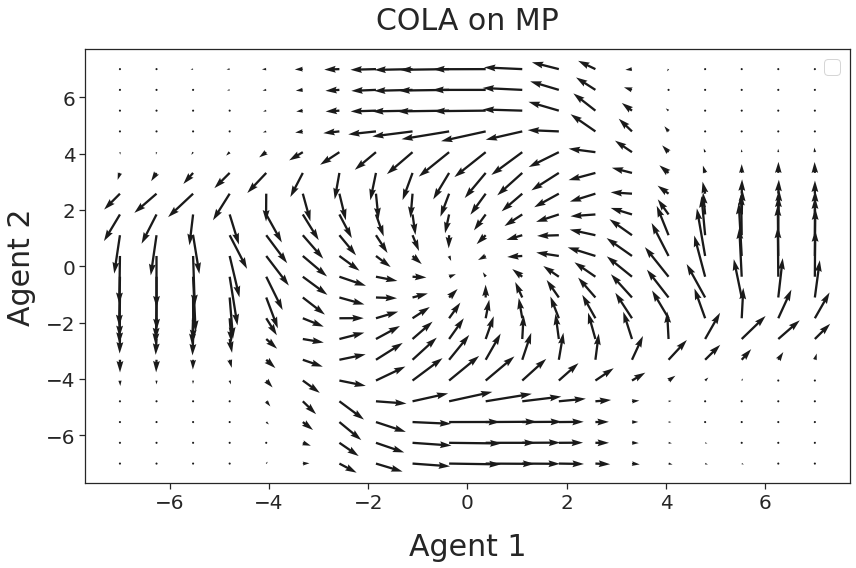}
  	  	\caption{}
  	\label{fig:imp_grad_field2}
 \end{subfigure}
\caption{Training in MP at look-ahead rate of \(\alpha=10\). (a) Axes are on a log-scale. Shown is the mean variance and consistency over 10 different runs. Each run, COLA was retrained. COLA\_long was trained for 80k steps and COLA\_short for 800 steps. (b) LOLA and HOLA find non-convergent or even diverging solutions, while COLA's converge. (c) Gradient field learned by COLA on MP at a look-ahead rate of 10.}
\label{fig:var_cons}
\end{figure*}
\subsection{Theoretical results for COLA}
\label{theoretical-results-cola}
In this section, we provide some initial theoretical results for COLA's uniqueness and convergence behavior, using the Tandem game \citep{letcher_stable_2019} and the Hamiltonian game \citep{balduzzi_mechanics_2018} as examples. These are simple polynomial games, with losses given in Section~\ref{section-experiments}. Proofs for the following propositions can be found in Appendices~\ref{proof:cons_not_unqiue}, \ref{proof:cons_sfps} and \ref{appendix:hamiltonian_theory}, respectively.

First, we show that solutions to the consistency equations are in general not unique, even when restricting to linear update functions in the Tandem game. Interestingly, empirically, COLA does seem to consistently converge to similar solutions regardless (see Table \ref{tab:self_sim} in Appendix \ref{appendix:mp}).
\begin{prop}\label{consistency_not_unique}
Solutions to the consistency equations are not unique, even when restricted to linear solutions; more precisely, there exist several linear consistent solutions to the Tandem game.
\end{prop}
Second, we show that consistent solutions do not, in general, preserve SFPs, contradicting the hypothesis that LOLA's failure to preserve SFPs is due to its inconsistency (\citet{letcher2018thesis},~pp.~2,~26; \citet{letcher2018thesis}).
We experimentally confirm this result in Section~\ref{results}.
\begin{prop}\label{consistency_sfps}
Consistency does not imply preservation of SFPs: there is a consistent solution to the Tandem game with $\al=1$ that fails to preserve \textbf{any} SFP. Moreover, for any $\al > 0$, there are \textbf{no} linear consistent solutions to the Tandem game that preserve more than one SFP.
\end{prop}

Third, we show that COLA can have more robust convergence behavior than LOLA and SOS:
\begin{prop}\label{prop:hamiltonian_theory}
For any non-zero initial parameters and any $\al > 1$, LOLA and SOS have divergent iterates in the Hamiltonian game. By contrast, any linear solution to the consistency equations converges to the origin for any initial parameters and \textbf{any} look-ahead rate $\al > 0$; moreover, the speed of convergence strictly increases with $\al$.
\end{prop}

\section{Experiments}
\label{section-experiments}
We perform experiments on a set of games from the literature \citep{balduzzi_mechanics_2018, letcher_stable_2019} using LOLA, SOS and CGD as baselines.
For details on the training procedure of COLA, we refer the reader to Appendix~\ref{appendix:cola_training}.

First, we compare HOLA and COLA on polynomial general-sum games, including the Tandem game \citep{letcher_stable_2019}, where LOLA fails to converge to SFPs. Second, we investigate non-polynomial games, specifically the zero-sum Matching Pennies (MP) game, the general-sum Ultimatum game \citep{hutter2020learning} and the IPD \citep{Axelrod84, Harper_2017}.


\paragraph{Polynomial games.}
\label{experiments-tandem-game}
Losses in the {\bf Tandem game} \citep{letcher_stable_2019} are given by $L^{1}(x,y)=(x+y)^{2}-2 x$ and $L^{2}(x, y)=(x+y)^{2}-2 y$ for agent 1 and 2 respectively.
The Tandem game was introduced to show that LOLA fails to preserve SFPs at $x + y = 1$ and instead converges to Pareto-dominated solutions \citep{letcher_stable_2019}. Additionally to the Tandem game, we investigate the algorithms on the {\bf Hamiltonian game}, \(L^{1}(x, y)=xy\) and \(L^{2}(x, y)=-xy\); and the {\bf Balduzzi game}, where \(L^{1}(x,y)=\frac{1}{2}x^2+10xy\) and \(L^{2}(x,y)=\frac{1}{2}y^2-10xy\) \citep{balduzzi_mechanics_2018}.

\begin{figure*}[hbt!]
 \centering
 \begin{subfigure}[]{0.32\linewidth}
     \centering
	\includegraphics[width=\linewidth]{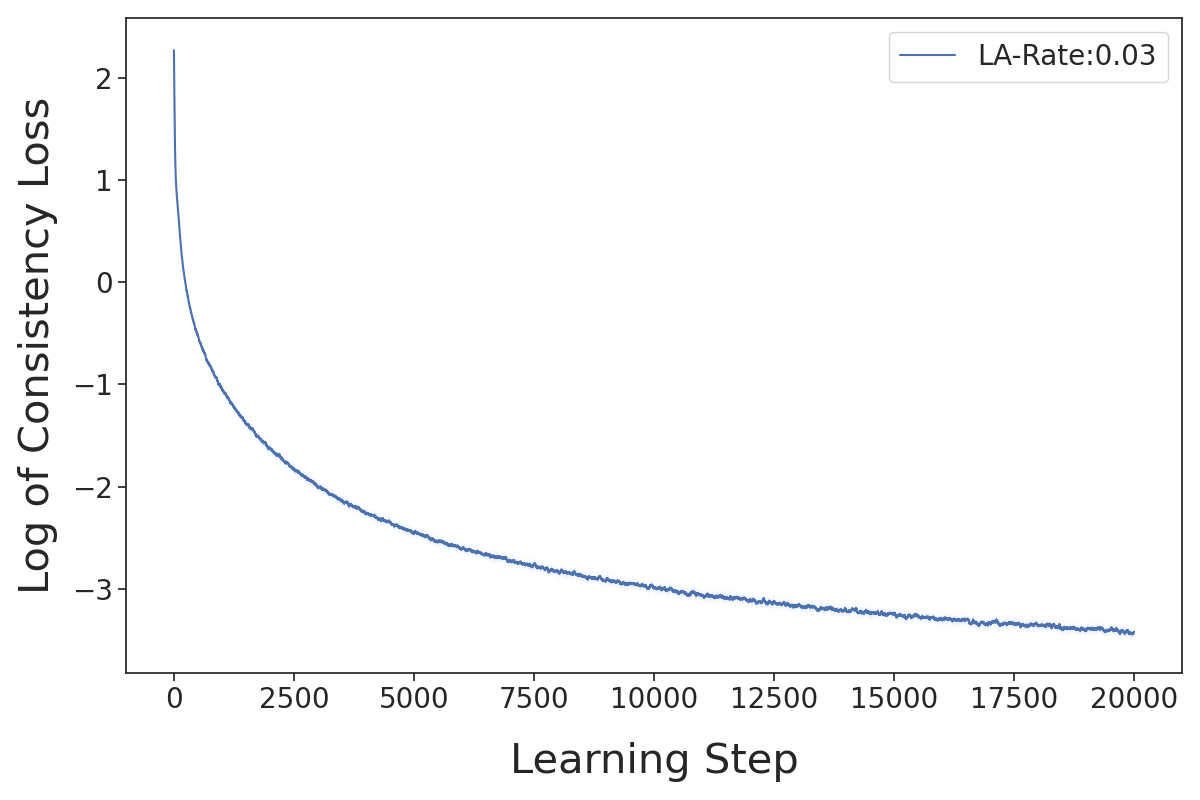}
	\caption{}
	\label{fig:ipd_consloss_low}
 \end{subfigure}
 \begin{subfigure}[]{0.32\linewidth}
    \centering
	\includegraphics[width=\linewidth]{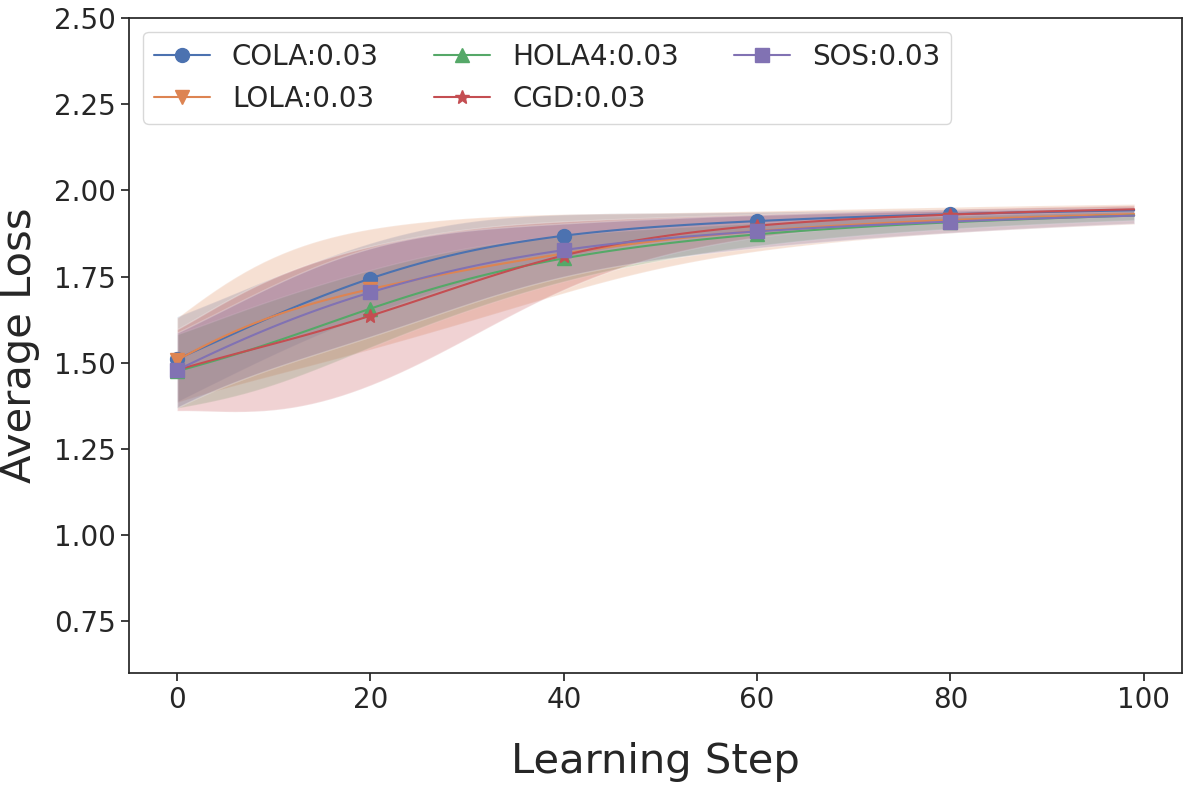}
	\caption{}
	\label{fig:ipd_play_low}
 \end{subfigure}
 \begin{subfigure}[]{0.32\linewidth}
	\includegraphics[width=\linewidth]{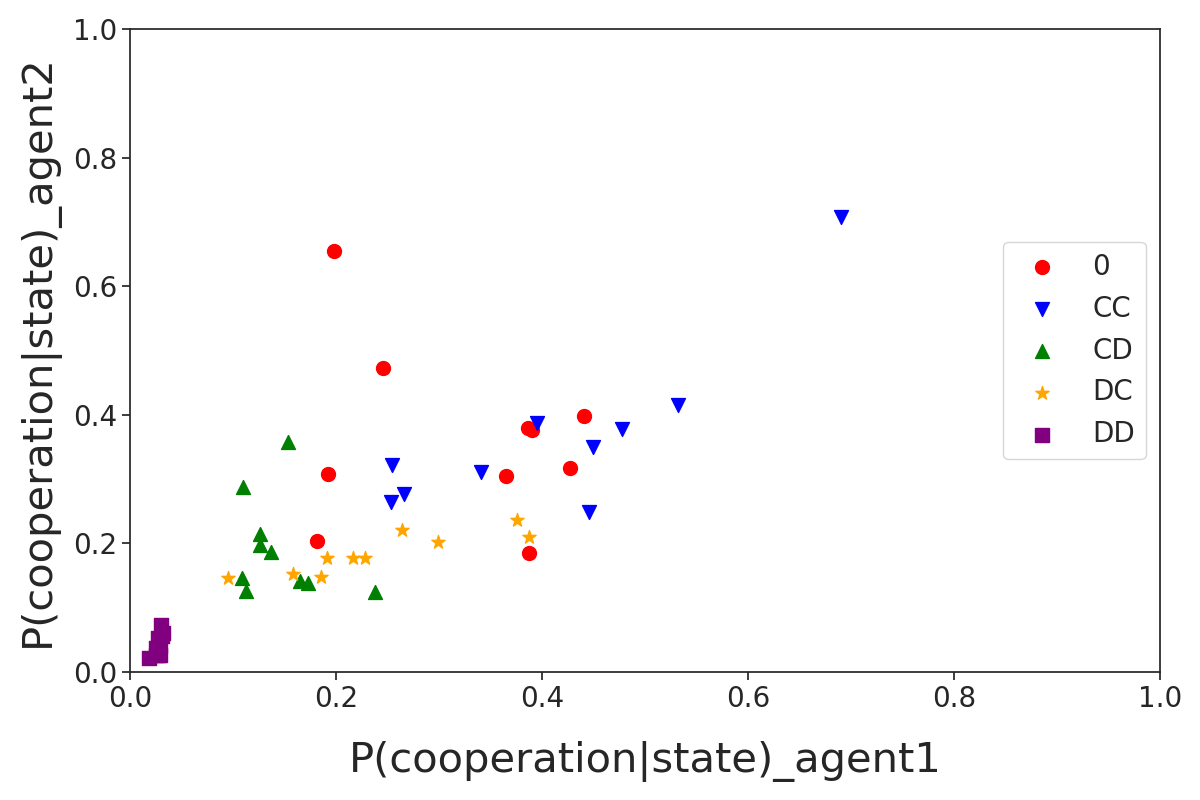}
	\caption{}
	\label{fig:ipd_policy_low}
 \end{subfigure}
 \begin{subfigure}[]{0.32\linewidth}
     \centering
	\includegraphics[width=\linewidth]{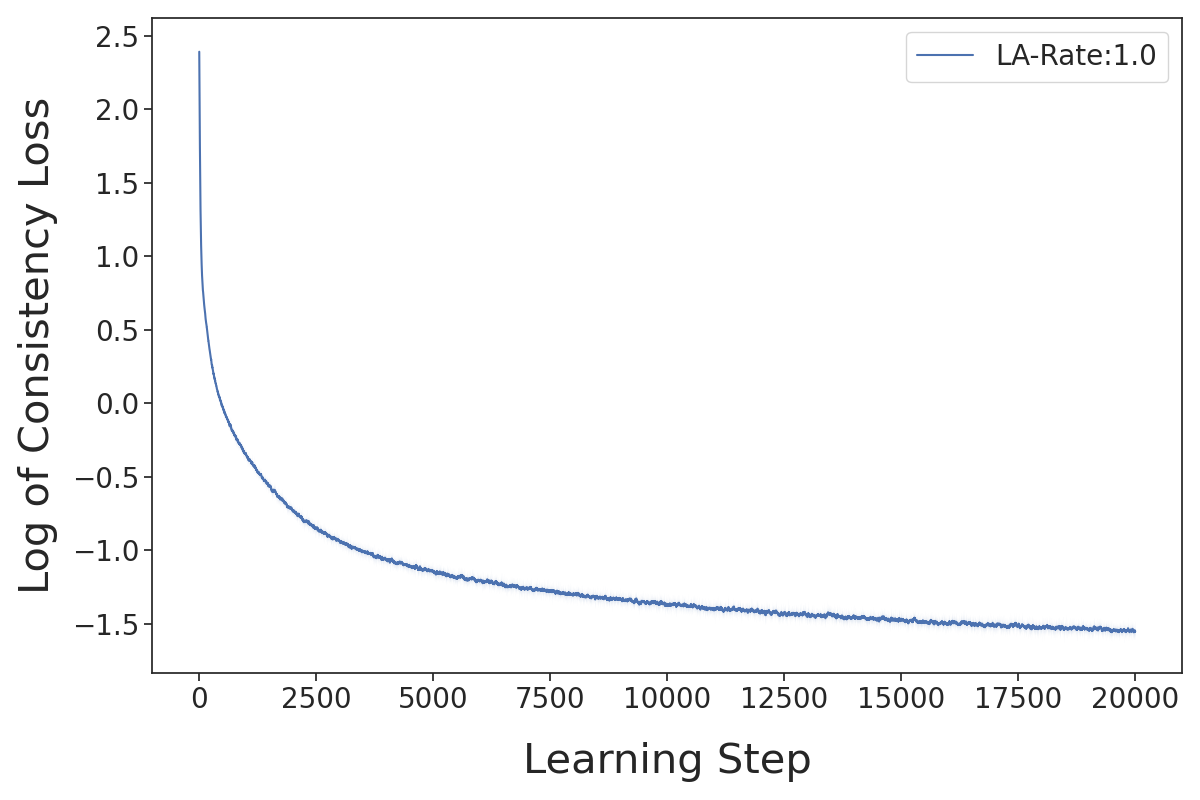}
	\caption{}
	\label{fig:ipd_consloss}
 \end{subfigure}
 \begin{subfigure}[]{0.32\linewidth}
    \centering
	\includegraphics[width=\linewidth]{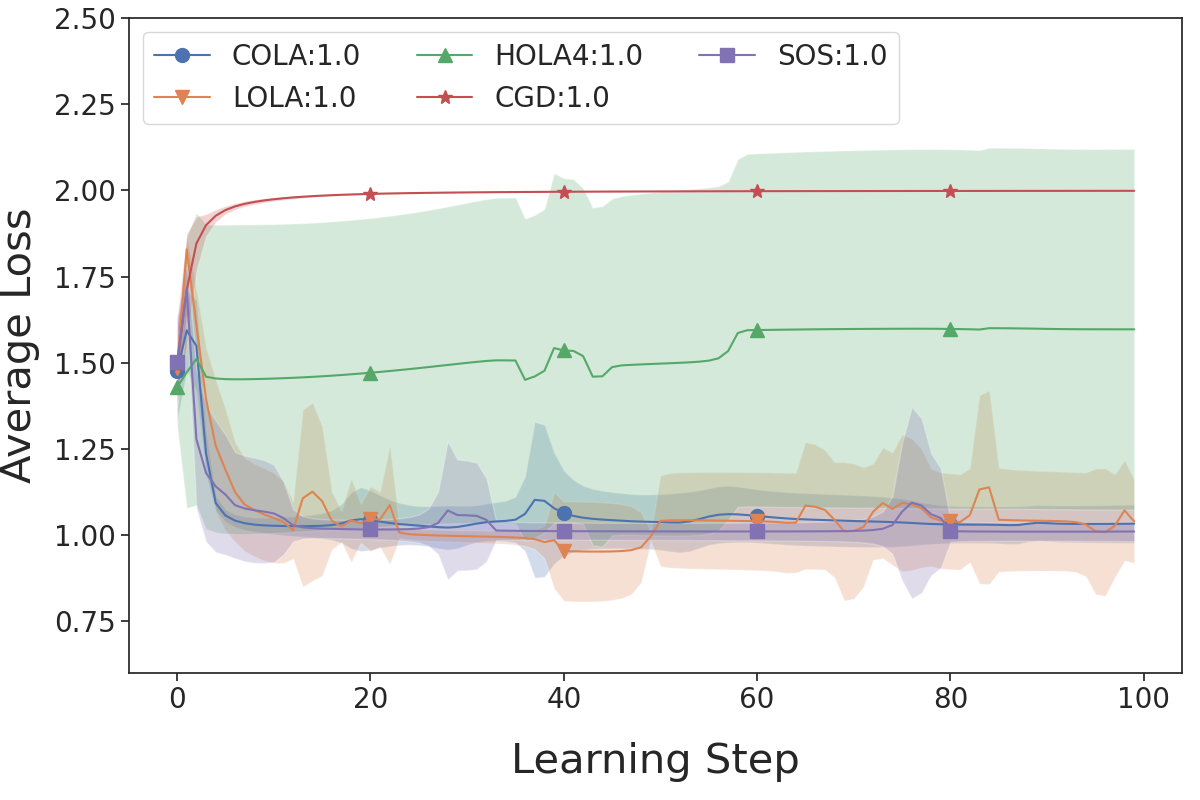}
	\caption{}
	\label{fig:ipd_play}
 \end{subfigure}
 \begin{subfigure}[]{0.32\linewidth}
	\includegraphics[width=\linewidth]{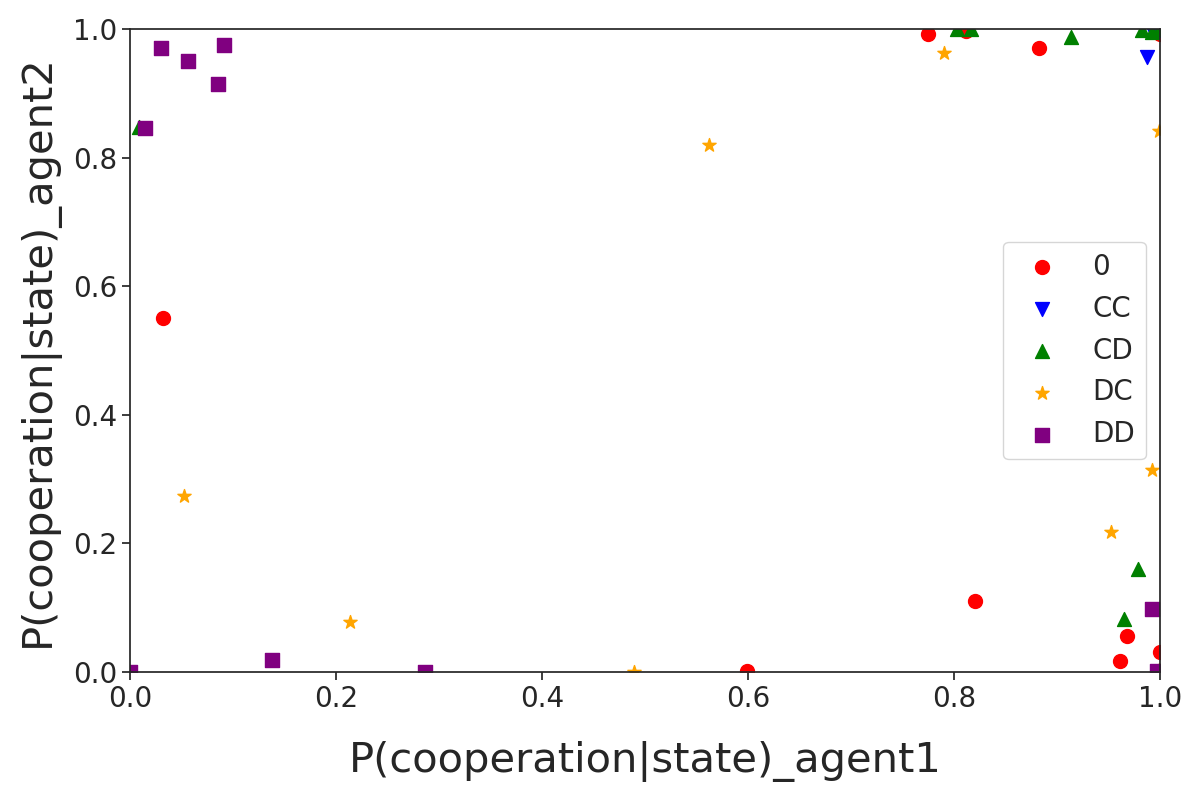}
	\caption{}
	\label{fig:ipd_policy}
 \end{subfigure}
  \caption{IPD results: Subfigure (a) / (d) show the consistency loss with shaded standard error over 10 independent training runs for look-ahead rates of 0.03 / 1.0, (b) / (e) the average loss and (c) / (f) the policy for the first player, both for the same pair of look-ahead rates. At low look-ahead HOLA defects and at high ones it diverges, also leading to high loss.} 
 \label{fig:ipd_game}
\end{figure*}

\paragraph{Matching Pennies.}
The payoff matrix for MP \citep{lee_application_1967} is shown in Appendix \ref{appendix:mp} in Table \ref{tab:imp}. Each policy is parameterized with a single parameter, the log-odds of choosing heads $p_{\text{heads}} = \sigma(\theta_{A})$. In this game, the unique Nash equilibrium is playing heads half the time.


\paragraph{Ultimatum game.}
The binary, single-shot Ultimatum game \citep{guth_experimental_1982, sanfey_neural_2003, oosterbeek_cultural_2004, henrich_foundations_2006} is set up as follows. Player 1 has access to $\$10$. They can split the money fairly with player 2 (\(\$5\) for each player) or they can split it unfairly (\(\$8\) for player 1, \(\$2\) for player 2). Player 2 can either accept or reject the proposed split. If player 2 rejects, the reward is 0 for both players. If player 2 accepts, the reward follows the proposed split. Player 1's parameter is the log-odds of proposing a fair split $p_{\text{fair}} = \sigma(\theta_{1})$. Player 2's parameter is the log-odds of accepting the unfair split (assuming that player 2 always accepts fair splits) $p_{\text{accept}} = \sigma(\theta_{2})$. In terms of losses, we have
\begin{align*}L_1 &= -\left(5 p_{\text{fair}}+8(1-p_{\text{fair}})p_{\text{accept}}\right)\\L_{2} &= -\left(5 p_{\text{fair}}+2(1-p_{\text{fair}})p_{\text{accept}}\right).\end{align*}

\paragraph{IPD.}
We investigate the infinitely iterated Prisoner's Dilemma (IPD) \citep{Axelrod84, Harper_2017} with discount factor $\gamma=0.96$ and the usual payout function (see Appendix \ref{appendix:ipd}). An agent $i$ is defined through 5 parameters, the log-odds of cooperating in the first time step and across each of the four possible tuples of past actions of both players in the later steps. 

\section{Results}
\label{results}
First, we report and compare the learning outcomes achieved by COLA and our baselines. We find that COLA update functions converge even under high look-ahead rates and learn socially desirable solutions. 
We also confirm our theoretical result (Proposition~\ref{cgd-ilola}) that \textit{CGD does not equal iLOLA}, contradicting \citeauthor{schafer_competitive_2020} (\citeyear{schafer_competitive_2020}), and that COLA does not, in general, maintain SFPs (Proposition~\ref{consistency_sfps}), contradicting the prior belief that this shortcoming is caused by inconsistency.

Second, we provide a more in-depth empirical analysis and comparison of the COLA and HOLA update functions, showing that COLA and HOLA tend to coincide when the latter converges, and that COLA is able to find consistent solutions even when HOLA diverges. Moreover, while COLA's solutions are not unique in theory (Proposition~\ref{consistency_not_unique}), we empirically find that in our examples COLA tends to find similar solutions across different independent training runs. 
Additional results supporting the above findings are reported in Appendix~\ref{appendix-further-results}.

\paragraph{Learning Outcomes.}


\begin{table*}[hbt!]
\small
\centering
\caption{IPD: Over multiple look-ahead rates we compare (a) the consistency losses and (b) the cosine similarity between COLA and LOLA, HOLA2, and HOLA4. The values represent the mean over 250 samples, uniformly sampled from the parameter space $\Theta$. The error bars represent one standard deviation and capture the variance over 10 different COLA training runs.}
\quad
\subfloat[]{\begin{tabular}{l|l l l l}
\label{tab:ipdcons}
$\alpha$ & LOLA    & HOLA2   & HOLA4   & \multicolumn{1}{c}{COLA} \\ \hline
1.0      & 39.56   & 21.16   & 381.21  & 1.06$\pm$0.09 \\ \hline
0.03     & 2e-3 & 5e-6 & 9e-8 & 0.16$\pm$0.02 \\ \hline
\end{tabular}
}
\quad
\subfloat[]{\begin{tabular}{l|l l l}
\label{tab:ipdsim}
$\alpha$ & LOLA & HOLA2 & HOLA4 \\ \hline
1.0      & 0.73$\pm$0.02 & 0.63$\pm$0.01  & 0.46$\pm$0.03  \\ \hline
0.03     & 0.97$\pm$0.01 & 0.97$\pm$0.01  & 0.97$\pm$0.01  \\ \hline
\end{tabular}
}
\end{table*}

In the Tandem game (Figure \ref{fig:tandem_play}), we see that COLA and HOLA8 converge to similar outcomes in the game, whereas CGD does not. This supports our theoretical result that CGD does not equal iLOLA (Proposition~\ref{cgd-ilola}). We also see that COLA does not recover SFPs, thus experimentally confirming Proposition~\ref{consistency_sfps}. In contrast to LOLA, HOLA and SOS, COLA finds a convergent solution even at a high look-ahead rate (see COLA:0.8 in Figure~\ref{fig:tandem_play} and Figure \ref{fig:tandem_play_div} in Appendix~\ref{appendix:tandem}). CGD is the only other algorithm in the comparison that also shows robustness to high look-ahead rates in the Tandem game. 

On the IPD, all algorithms find the defect-defect strategy on low look-ahead rates (Figure \ref{fig:ipd_play_low}). At high look-ahead rates, COLA finds a strategy qualitatively similar to tit-for-tat, as displayed in Figure \ref{fig:ipd_policy}, though more noisy. However, COLA still achieves close to the optimal total loss, in contrast to CGD, which finds defect-defect even at a high look-ahead rate (see Figure \ref{fig:cgd_on_ipd} in Appendix \ref{appendix:ipd}). The fact that, unlike HOLA and COLA, CGD finds defect-defect, further confirms that CGD does not equal iLOLA.

On MP at high look-ahead rates, SOS and LOLA mostly don’t converge, whereas COLA converges\textit{ even faster }with a high look-ahead rate (see Figure \ref{fig:var_cons_high}), confirming Proposition \ref{prop:hamiltonian_theory} experimentally (also see Figure \ref{fig:bal_play} and \ref{fig:ham_play} in Appendix~\ref{appendix:hamiltonian}). To further investigate the influence of consistency on learning behavior, we plot the consistency of an update function against the variance of the losses across learning steps achieved by that function, for different orders of HOLA and for COLA (Figure~\ref{fig:mp_div}). At a high look-ahead rate in Matching Pennies, we find that more consistent update functions tend to lead to lower variance across training, demonstrating a potential benefit of increased consistency at least at high look-ahead rates.

For the Ultimatum game, we find that COLA is the only method that finds the fair solution consistently at a high look-ahead rate, whereas SOS, LOLA, and CGD do not (Figure \ref{fig:ult_play}). At low look-ahead rates, all algorithms find the unfair solution (see Figure \ref{fig:ult_play_low} in Appendix \ref{appendix:ult}). This demonstrates an advantage of COLA over our baselines and shows that higher look-ahead rates can lead to better learning outcomes.

Lastly, we introduce the Chicken game in Appendix \ref{appendix:chicken_game}. Both Taylor LOLA and SOS crash, whereas COLA, HOLA, CGD, and exact LOLA swerve at high look-ahead rates (Figure \ref{fig:chicken_play_high}). Crashing in Chicken results in a catastrophic payout for both agents, whereas swerving results in a jointly preferable outcome.\footnote{Interestingly, in contrast to Taylor LOLA, exact LOLA swerves. The Chicken game is the only game where we found a difference in learning behavior between exact LOLA and Taylor LOLA.}


\paragraph{Update functions.}


Turning to our analysis of COLA and HOLA update functions, we first investigate how increasing the order of HOLA affects the consistency of its updates. As shown in Table \ref{tab:beta_hola}, \ref{tab:imp_cons} and \ref{tab:ipdcons}, HOLA's updates become more consistent with increasing order, but only below a certain, game-specific look-ahead rate threshold. Above that threshold, HOLA's updates become less consistent with increasing order. 

Second, we compare the consistency losses of COLA and HOLA. In the aforementioned tables, we observe that COLA achieves low consistency losses on most games. Below the threshold, COLA finds similarly low consistency losses as HOLA, though there HOLA's are lower in the non-polynomial games. Above the threshold, COLA finds consistent updates, even when HOLA does not. A visualization of the update function learned by COLA at a high look-ahead rate on the MP is given in Figure \ref{fig:imp_grad_field2}.

For the IPD, COLA's consistency losses are high compared to other games, but much lower than HOLA's consistency losses at high look-ahead rates. We leave it to future work to find methods that obtain more consistent solutions. 

Third, we are interested whether COLA and HOLA find similar solutions. We calculate the cosine similarity between the respective update functions over $\Theta$. As we show in Table \ref{tab:tandem_cos}, \ref{tab:imp_sim} and \ref{tab:ipdsim}, COLA and HOLA find very similar solutions when HOLA's updates converge, i.e., when the look-ahead rate is below the threshold. Above the threshold, COLA's and HOLA's updates unsurprisingly become less similar with increasing order, as HOLA's updates diverge with increasing order.

Lastly, we investigate Proposition \ref{consistency_not_unique} empirically and find that COLA finds similar solutions in Tandem and MP over 5 training runs (see Table \ref{tab:self_sim} in Appendix \ref{appendix:mp}). Moreover, the small standard deviations in Table \ref{tab:ipdsim} indicate that COLA also finds similar solutions over different runs in the IPD.

\section{Conclusion and Future Work}
In this paper, we corrected a claim made in prior work \citep{schafer_competitive_2020}, clearing up the relation between the CGD and LOLA algorithms. We also showed that iLOLA solves part of the consistency problem of LOLA. We introduced COLA, which finds consistent solutions without requiring many recursive computations like iLOLA. It was believed that inconsistency leads to arrogant behaviour and lack of preservation of SFPs. We showed that even with consistency, opponent shaping behaves \textit{arrogantly}, pointing towards a fundamental open problem for the method.

In a set of games, we found that COLA tends to find prosocial solutions. Although COLA's solutions are not unique in theory, empirically, COLA tends to find similar solutions in different runs. It coincides with iLOLA when HOLA converges and finds consistent update functions even when HOLA fails to converge with increasing order. Moreover, we showed empirically (and in one case theoretically) that COLA update functions converge under a wider range of look-ahead rates than HOLA and LOLA update functions.

This work raises many questions for future work, such as the existence of solutions to the COLA equations in general games and general properties of convergence and learning outcomes. 
Moreover, additional work is needed to scale COLA to large settings such as GANs or Deep RL, or settings with more than two players. Another interesting axis is addressing further inconsistent aspects of LOLA as identified in \citet{letcher_stable_2019}.

\section*{Acknowledgements}
 A part of this work was done while Timon Willi and Jakob Foerster were at the Vector Institute, University of Toronto. They are grateful for the access to the Vector Institute's compute infrastructure. They are also grateful for the access to the Advanced Research Computing (ARC) infrastructure. Johannes Treutlein is grateful for support by the Center on Long-Term Risk. 

\bibliography{references}
\bibliographystyle{icml2022}

\newpage
\appendix
\onecolumn


\section{Nonconvergence of HOLA in the Tandem Game}
\label{appendix:nonconvergence}
In the following, we show that for the choice of look-ahead rate \(\alpha=1\), HOLA does not converge in the Tandem game. This shows that given a large enough look-ahead rate, even in a simple quadratic game, HOLA need not converge.
\begin{prop}
Let \(L^1,L^2\) be the two players' loss functions in the Tandem game as defined in Section~\ref{experiments-tandem-game}:
\begin{equation}
    L^{1}(x,y)=(x+y)^{2}-2 x \quad \text{and} \quad L^{2}(x, y)=(x+y)^{2}-2y,
\end{equation}
and let \(h^n_i\) denote the \(n\)-th order exact LOLA update for player \(i\) (where \(n=0\) denotes naive learning). Consider the look-ahead rate \(\alpha:=1\). Then the functions \((h^n_i)_{n\in\mathbb{N}}\) for \(i=1,2\) do not converge pointwise.
\end{prop}
\begin{proof}
We will prove the auxiliary statement that
\[h_i^n(x,y)=2^{n+2}-2(1+x+y)\]
for \(i=1,2\).
It then follows trivially that the \(h_i^n\) cannot converge.

The auxiliary result can be proven by induction. The base case \(n=0\) follows from \[{\nabla_iL^i(x,y)=2-2(x+y)=2^2-2(1+x+y)}\] 
for \(i=1,2\). Next, for the inductive step, we have to show that
\begin{align}
h_1^{n}(x,y)&=-\nabla_1( L^1(x,y+h_{2}^{n-1}(x,y)))\\
h_2^{n}(x,y)&=-\nabla_2( L^2(x+h_{1}^{n-1}(x,y),y))
\end{align}
for any \(n>0\). Substituting the inductive hypothesis in the second step, we have
\begin{align}&\phantom{=}-\nabla_1( L^1(x,y+h_{2}^{n-1}(x,y)))\\
&=-\nabla_1( L^1(x,y+2^{n+1}-2(1+x+y)))\\
&=-\nabla_1\left((x+y+2^{n+1}-2-2x-2y)^2 -2x\right)\\
&=-\nabla_1\left((-x-y+2^{n+1}-2)^2 -2x\right)\\
&=2(-x-y+2^{n+1}-2) +2\\
&=2^{n+2} - 2(1+x+y)\\
&=h_1^n(x,y).
\end{align}
The derivation for \(h_2^n(x,y)\) is exactly analogous.
This shows the inductive step and thus finishes the proof.

\end{proof}

\section{Proof of Proposition~\ref{ilola-consistent}}
\label{appendix:ilolaconv}
%
To begin, recall that some differentiable game with continuously differentiable loss functions \(L^1,L^2\) is given, and that \(h^n=(h_1^n,h_2^n)\) denotes the \(n\)-th order exact LOLA update function. We assume that the iLOLA update function \(h\) exists, defined via
\[h_i(\theta):=
\lim_{n\rightarrow\infty}h^n_i(\theta),
\]
for all \(\theta\in\mathbb{R}^d\).

To prove Proposition~\ref{ilola-consistent}, we need to show that \(h_1,h_2\) are consistent, i.e., satisfy Definition~\ref{defn-consistency}, under the assumption that 

\[\lim_{n\rightarrow\infty}\nabla_ih_{-i}^n(\theta)=
\nabla_ih_{-i}(\theta)
\]
for \(i=1,2\) and any \(\theta\).

To that end, define the \emph{(exact) LOLA operator} \(\Psi\) as the function mapping a pair of update functions \(f:=(f_1,f_2)\) to the RHS of Equations~\ref{eq:cons1} and \ref{eq:cons2},
\begin{align}\Psi_1(f)(\theta)&:=-\alpha\nabla_1(L^1(\theta_1,\theta_2 + f_2(\theta_1,\theta_2)))
\\
\Psi_2(f)(\theta)&:=-\alpha\nabla_2(L^2(\theta_1+\ f_1(\theta_1,\theta_2),\theta_2))
\end{align}
for any \(\theta\).
Note that then we have \(h_i^{n+1}=\Psi_i(h^n)\), i.e., \(\Psi\) maps \(n\)-th order LOLA to \(n+1\)-order LOLA.

In the following, we show that iLOLA is a fixed point of the LOLA operator, i.e., \(\Psi(h)=h\). It follows from the definition of \(\Psi\) that then \(h\) is consistent. We denote by \(\Vert\cdot\Vert\) the Euclidean norm or the induced operator norm for matrices. We focus on showing \(\Psi_1(h)=h_1\). The case \(i=2\) is exactly analogous.

For arbitrary \(\theta\) and \(n\), define \(\hat{\theta}_2:=\theta_2 + h_2(\theta)\) and \(\hat{\theta}_2^n:=\theta_2 + h^2_n(\theta)\) as the updated parameter of player \(2\). First, it is helpful to show that \(\Psi_1(h^n)(\theta)\) converges to \(\Psi_1(h)(\theta)\):
\begin{align}
0
&\leq\Vert \Psi_1(h)(\theta)-\Psi_1(h^n)(\theta)\Vert\\
&=\alpha
\Vert \nabla_1 (L^1(\theta_1,\theta_2 + h_2(\theta)))-\nabla_1 (L^1(\theta_1,\theta_2 + h_2^n(\theta))\Vert\\
&=\alpha
\Vert (\nabla_1h_2(\theta))^\top \nabla_2L^1(\theta_1,\hat{\theta}_2) - (\nabla_1h_2^n(\theta))^\top \nabla_2L^1(\theta_1,\hat{\theta}_2^n)
+ \nabla_1 L^1(\theta_1, \hat{\theta}_2)-\nabla_1 L^1(\theta_1, \hat{\theta}_2^n)\Vert
\\
&\leq
\alpha\Vert (\nabla_1h_2(\theta))^\top \nabla_2L^1(\theta_1,\hat{\theta}_2) - (\nabla_1h_2^n(\theta))^\top \nabla_2L^1(\theta_1,\hat{\theta}_2^n)\Vert + \alpha\Vert\nabla_1 L^1(\theta_1, \hat{\theta}_2)-\nabla_1 L^1(\theta_1, \hat{\theta}_2^n)\Vert
\\
&=
\alpha\Vert (\nabla_1h_2(\theta))^\top (\nabla_2L^1(\theta_1,\hat{\theta}_2) - \nabla_2L^1(\theta_1,\hat{\theta}_2^n))
\\
&\quad+\nonumber
 (\nabla_1h_2(\theta)-\nabla_1h_2^n(\theta))^\top\nabla_2L^1(\theta_1,\hat{\theta}_2^n)\Vert
+ \alpha\Vert\nabla_1 L^1(\theta_1, \hat{\theta}_2)-\nabla_1 L^1(\theta_1,\hat{\theta}_2^n)\Vert
\\
&\leq\nonumber
\alpha\Vert (\nabla_1h_2(\theta))^\top \Vert\Vert\nabla_2L^1(\theta_1,\hat{\theta}_2) - \nabla_2L^1(\theta_1,\hat{\theta}_2^n)\Vert
\\
&\quad+
\alpha\Vert (\nabla_1h_2(\theta)-\nabla_1h_2^n(\theta))^\top\Vert\Vert\nabla_2L^1(\theta_1,\hat{\theta}_2^n)\Vert
+ \alpha\Vert\nabla_1 L^1(\theta_1, \hat{\theta}_2)-\nabla_1 L^1(\theta_1,\hat{\theta}_2^n)\Vert\label{eq:123}
\\
&\overset{n\rightarrow\infty}{\longrightarrow} 0.
\end{align}
In the last step, we used the following two facts. First, since \(\nabla_iL^1(\theta)\) is assumed to be continuous in \(\theta_2\), and \(\lim_{n\rightarrow\infty}\hat{\theta}_2^n = \theta_2 + \lim_{n\rightarrow\infty}h_2^n(\theta)=\theta_2 + h_2(\theta)=\hat{\theta}_2\) by assumption, it follows that \(\lim_{n\rightarrow \infty}\nabla_2L^1(\theta_1,\hat{\theta}_2^n)=\nabla_2 L^1(\theta_1,\hat{\theta}_2)\) and \(\lim_{n\rightarrow \infty}\nabla_1L^1(\theta_1,\hat{\theta}_2^n)=\nabla_1 L^1(\theta_1,\hat{\theta}_2)\). Second, by assumption, \(\lim_{n\rightarrow\infty}\nabla h_2^n(\theta)=\nabla h_2(\theta)\). In particular, \(\Vert\nabla_2L^1(\theta_1,\hat{\theta}_2^n)\Vert\) must be bounded, and thus the three terms in (\ref{eq:123}) must all converge to \(0\) as \(n\rightarrow\infty\). It follows by the sandwich theorem that \(\lim_{n\rightarrow\infty}\Psi_1(h^n)(\theta)=\Psi_1(h)(\theta)\).

Now we can directly prove that \(\Psi_1(h)(\theta)=h_1(\theta)\). It is
\begin{align}\label{eq:sandwich1}0&\leq 
    \Vert\Psi_1(h)(\theta)-h_1(\theta)\Vert
    \\
       & = \Vert \Psi_1(h)(\theta)-\Psi_1(h^n)(\theta)
    + \Psi_1(h^n)(\theta) - h_1^n(\theta)
    +  h_1^n(\theta) - h_1(\theta)\Vert\\
   & \leq \Vert \Psi_1(h)(\theta)-\Psi_1(h^n)(\theta)\Vert
    + \Vert \Psi_1(h^n)(\theta) - h_1^n(\theta)\Vert
    + \Vert h_1^n(\theta) - h_1(\theta)\Vert\\
    &= \Vert \Psi_1(h)(\theta)-\Psi_1(h^n)(\theta)\Vert
    + \Vert h_1^{n+1}(\theta) - h_1^n(\theta)\Vert
    + \Vert h_1^n(\theta) - h_1(\theta)\Vert\\
    &\overset{n\rightarrow\infty}{\longrightarrow} 0\label{eq:sandwich2}
,\end{align}
where in the last step we have used the above result, as well as the assumption that \(h^n_1(\theta)\) converges pointwise, and thus must also be a Cauchy sequence, so the last and the middle term both converge to zero as well.

It follows by the sandwich theorem that \(\Psi_1(h)(\theta)=h_1(\theta)\). Since \(\theta\) was arbitrary, this concludes the proof.\hfill \qedsymbol

\section{Infinite-order Taylor LOLA}\label{appendix-Taylor-iLOLA}
In this Section, we repeat the analysis of iLOLA from Section~\ref{conv-and-consistency-LOLA} for infinite-order Taylor LOLA (Taylor iLOLA). I.e., we define Taylor consistency, and show that Taylor iLOLA satisfies this consistency equation under certain assumptions. This result will be needed for our proof of Proposition~\ref{cgd-ilola}.

To begin, assume that some differentiable game with continuously differentiable loss functions \(L^1,L^2\) is given. Define the Taylor LOLA operator \(\Phi\) that maps pairs of update functions \((f_1,f_2)\) to the associated Taylor LOLA update
\begin{equation}\label{itlola-operator}\Phi_i(f)
:= -\alpha \nabla_i ( L^i + (\nabla_{-i} L^i)^\top f_{-i} )
\end{equation}
for \(i=1,2\).

We then have the following definition.
\begin{defn}[Taylor consistency]Two update functions \(f_1,f_2\) are called Taylor consistent if for any \(i=1,2\), we have
\[\Phi(f_1,f_2)=(f_1,f_2).\]
\end{defn}

Next, let \(h^n_i\) denote \(i\)'s \(n\)-th order Taylor LOLA update. I.e., \(h_i^n:=\Phi_i(h^{n-1})\) for \(n\geq 0\), where we let \(h^{-1}_i:=0\). Then we define
\begin{defn}[Taylor iLOLA] 
If $(h_1^n, h_2^n)$ converges pointwise as $n \to \infty$, define Taylor iLOLA as the limiting update
\[h \coloneqq \lim_{n\to\infty} \begin{pmatrix}
h_1^n \\ h_2^n.
\end{pmatrix} \]
\end{defn}

Finally, we provide a proof that Taylor iLOLA is Taylor consistent; i.e., we give a Taylor version of Proposition~\ref{ilola-consistent}.

\begin{prop}\label{itlola-consistent}Let \(h^n_i\) denote player \(i\)'s \(n\)-th order Taylor LOLA update. 
Assume that ${\lim_{n\rightarrow\infty}h_i^n(\T)=h_i(\T)}$ and ${\lim_{n\rightarrow\infty}\nabla_{i} h_{-i}^n(\T)=\nabla_ih_{-i}(\T)}$ for all $\T$ and $i \in \{1, 2\}$. Then Taylor iLOLA is Taylor consistent.
\end{prop}
\begin{proof}
The proof is exactly analogous to that of Proposition~\ref{ilola-consistent}, but easier.
We show \(\Phi(h)=h\). It follows from the definition of \(\Phi\) in Equation~\ref{itlola-operator} that then \(h\) is Taylor consistent. We focus on showing \(\Phi_1(h)=h_1\), and the case \(i=2\) is exactly analogous.

First, we show that \(\Phi_1(h^n)(\theta)\) converges to \(\Phi_1(h)(\theta)\) for all \(\theta\). Letting \(n\) be arbitrary and omitting \(\theta\) in the following for clarity, it is
\begin{align}0&\leq\Vert \Phi_1(h)-\Phi_1(h^n)\Vert\\
&=
\Vert -\alpha \nabla_1 ( L^1 + (\nabla_{2} L^1)^\top h_2 )+\alpha \nabla_1 ( L^1 + (\nabla_{2} L^1)^\top h^n_2 )\Vert\\
&=\alpha
\Vert -\nabla_{12} L^1 h_2  -(\nabla_2L^1)^\top(\nabla_1h_2)^\top + \nabla_{12} L^1h^n_2 + (\nabla_2L^1)^\top(\nabla_1h_2^n)^\top \Vert\\
&\leq\alpha
\Vert\nabla_{12} L^1(h^n_2- h_2)\Vert+\alpha\Vert
(\nabla_2L^1)^\top(\nabla_1h_2^n -\nabla_1h_2)^\top\Vert\\
&\leq\alpha
\Vert\nabla_{12} L^1\Vert \Vert h^n_2 -h_2\Vert +\alpha \Vert\nabla_2L^1\Vert\Vert\nabla_1h_2^n -\nabla_1h_2\Vert\\
&\overset{n\rightarrow\infty}{\longrightarrow} 0.
\end{align}
In the last step, we used the assumptions that \(\lim_{n\rightarrow\infty}h^n_2=h_2\) and \(\lim_{n\rightarrow\infty}\nabla h^n_2 = \nabla h_2\). 
It follows by the sandwich theorem that \(\lim_{n\rightarrow\infty}\Phi_1(h^n)(\theta)=\Phi_1(h)(\theta)\).

It follows from the above that \(\Phi_1(h)(\theta)=h_1(\theta)\), using exactly the same argument as in Equations~\ref{eq:sandwich1}-\ref{eq:sandwich2} with \(\Phi\) instead of \(\Psi\). Since \(\theta\) was arbitrary, this concludes the proof.

\end{proof}

\section{Proof of Proposition \ref{cgd-ilola}}
\label{appendix:cgd-ilola}

We begin by proving that LCGD does not coincide with Taylor LOLA and CGD neither coincides with exact nor Taylor iLOLA. It is sufficient to manifest a single counter-example: we consider the Tandem game given by $L^1 = (x+y)^2-2x$ and $L^2 = (x+y)^2-2y$ (using $x,y$ instead of $\T_1, \T_2$ for simplicity). Throughout this proof we use the notation introduced by \citet{balduzzi_mechanics_2018} and \citet{letcher_stable_2019} including the \textit{simultaneous gradient}, the \textit{off-diagonal Hessian} and the \textit{shaping term} of the game as
\[ \xi = \begin{pmatrix}
\nabla_1 L^1 \\ \nabla_2 L^2
\end{pmatrix} \qquad \text{and} \qquad H_o = \begin{pmatrix}
0 & \nabla_{12} L^2 \\
\nabla_{21} L^2 &  0
\end{pmatrix} \qquad \text{and} \qquad \lchi = \diag \left(H_o^T \nabla L\right) \]
respectively. Note that in two-player games, Taylor LOLA's shaping term reduces to
\[ \lchi = \begin{pmatrix}
\nabla_{12} L^2 \nabla_2 L^1 \\ \nabla_{21} L^1 \nabla_1 L^2 
\end{pmatrix} \,. \]

\paragraph{LCGD $\neq$ LOLA.} Following \citet{schafer_competitive_2020}, LCGD is given by
\[ \text{LCGD} = -\al \begin{pmatrix}
\nabla_x f - \al D^2_{xy} f \nabla_y g \\
\nabla_y g - \al D^2_{yx} g \nabla_x f
\end{pmatrix} = -\al\begin{pmatrix}
I & - \al D^2_{xy} f \\
- \al D^2_{yx} g & I
\end{pmatrix} \begin{pmatrix}
\nabla_x f \\ \nabla_y g
\end{pmatrix} = -\al(I-\al H_o)\xi \]
while Taylor LOLA is given \citep{letcher_stable_2019} by
\begin{align*}
\text{LOLA} &= -\al(I-\al H_o)\xi + \al^2 \lchi \,.
\end{align*}
Any game with $\lchi \neq 0$ will yield a difference between LCGD and LOLA; in particular,
\[ \lchi = 4(x+y) \begin{pmatrix}
1 \\ 1
\end{pmatrix} \]
in the Tandem game implies that LCGD $\neq$ LOLA whenever parameters lie outside the measure-zero set $\{x+y = 0\} \subset \mathbb{R}^2$.

\paragraph{CGD does not recover HOLA.} Since CGD is obtained through a bilinear approximation (Taylor expansion) of the loss functions, one would expect that the authors' claim of recovering HOLA is with regards to Taylor (not exact) HOLA. For completeness, and to avoid any doubts for the reader, we prove that CGD neither corresponds to exact nor Taylor HOLA.

Following \citet{schafer_competitive_2020}, the series-expansion of CGD is given by
\[ \text{CGD}n =  -\al \sum_{i=0}^{n} \begin{pmatrix}
0 & -\al D^2_{xy} f \\ -\al D^2_{yx} g & 0
\end{pmatrix}^i \begin{pmatrix}
D_x f \\ D_y g
\end{pmatrix} = -\al \sum_{i=0}^{n} (-\al H_o)^i \xi \]
and converges to CGD whenever $\al < 1/\Vert H_o\Vert $ (where \(\Vert \cdot\Vert\) denotes the operator norm induced by the Euclidean norm on the space). Assume for contradiction that the series-expansion of CGD recovers HOLA, i.e. that CGD$n$ = HOLA$n$ for all $n$. In particular, we must have
\[ \text{CGD} = \lim_{n \to \infty} \text{CGD}n = \lim_{n \to \infty} \text{HOLA}n = \text{iLOLA} \]
whenever $\al < 1/\Vert  H_o \Vert$. In the tandem game, we have
\[ H_o = 2 \begin{pmatrix}
0 & 1 \\ 1 & 0
\end{pmatrix} \]
with $\Vert H_o \Vert = 2$, so CGD = iLOLA whenever $\al < 1/2$. Moreover, $H_o$ being constant implies that
\[ \nabla \text{HOLA}n = \nabla \text{CGD}n = -\al \sum_{i=0}^n (-\al H_o)^i \nabla \xi \,, \]
so gradients of HOLA also converge pointwise for all $\al < 1/2$. In particular, CGD = iLOLA must satisfy the (exact or Taylor) consistency equations by Proposition \ref{ilola-consistent} or Proposition \ref{itlola-consistent}. However, the update for CGD is given by
\[ \begin{pmatrix}
f_1 \\ f_2
\end{pmatrix} = -\al(I+\al H_o)^{-1}\xi = -2\al(x+y-1)\begin{pmatrix}
1 & 2\al \\ 2\al & 1
\end{pmatrix}^{-1} \begin{pmatrix}
1 \\ 1
\end{pmatrix} = \frac{-2\al(x+y-1)}{1+2\al}\begin{pmatrix}
1 \\ 1
\end{pmatrix} \,. \]

For the exact case, the RHS of the first consistency equation is
\begin{align*}
-\al \nabla_x \left( (x+y+f_2)^2-2x \right) &= -2\al\left((1+\nabla_x f_2)(x+y+f_2)-1\right) \\
&= \frac{-2\al}{1+2\al}\left(x+y+\frac{-2\al(x+y-1)}{1+2\al}-1-2\al \right) \\
&= f_1 + \frac{4\al^2(x+y+2\al)}{(1+2\al)^2}
\end{align*}
which does not coincide with the LHS of the consistency equation ($=f_1$) whenever parameters lie outside the measure-zero set $\{x+y+2\al = 0\} \subset \R^2$. Similarly for Taylor iLOLA, the RHS of the first consistency equation is
\begin{align*}
-\al \nabla_x \left( (x+y)^2-2x + 2(x+y)f_2 \right) &= -2\al \left( x+y-1+\frac{-2\al(x+y-1)}{1+2\al}+\frac{-2\al(x+y)}{1+2\al} \right) \\
&= f_1 + \frac{4\al^2(x+y)}{(1+2\al)^2}
\end{align*}
which does not coincide with the LHS of the consistency equation ($=f_1$) whenever parameters lie outside the measure-zero set $\{x+y = 0\} \subset \mathbb{R}^2$. This is a contradiction to consistency; we are done.

\paragraph{LCGD = LookAhead.} We have already shown that LCGD is given by $-\al(I-\al H_o)\xi$ in the proof that LCGD $\neq$ LOLA. This coincides exactly with LookAhead following \citet{letcher_stable_2019}.

\paragraph{CGD recovers higher-order LookAhead.} The series expansion of CGD is given by
\[ \text{CGD}n = -\al \sum_{i=0}^{n} (-\al H_o)^i \xi \,. \]
Also, higher-order (Taylor) LookAhead is defined recursively by expanding
\begin{align*}
f_1^{n+1} &= -\al \nabla_1 \left( L^1(\T^1, \T^2 + \bot f_2^{n}) \right) \approx -\al \left( \nabla_1 L^1 + \nabla_{12} L^1 f_2^{n} \right) \\
f_2^{n+1} &= -\al \nabla_2 \left( L^2(\T^1+\bot f_1^{n}, \T^2) \right) \approx -\al \left( \nabla_2 L^2 + \nabla_{21} L^2 f_1^{n} \right) \,,
\end{align*}
where $\bot$ is the stop-gradient operator (see \citep{balduzzi_mechanics_2018} for details on this operator) and $f_1^{-1} = f_2^{-1} = 0$. This can be written more succinctly as
\[ \begin{pmatrix}
f_1^{n+1} \\
f_2^{n+1}
\end{pmatrix} = -\al \begin{pmatrix}
\nabla_1 L^1 + \nabla_{12} L^1 f_2^{n}  \\ \nabla_2 L^2 + \nabla_{21} L^2 f_1^{n}
\end{pmatrix} = -\al \xi - \al H_o \begin{pmatrix}
f_1^{n} \\
f_2^{n}
\end{pmatrix} \,. \]
We prove by induction that
\[ \begin{pmatrix}
f_1^n \\ f_2^n
\end{pmatrix} = -\al \sum_{i=0}^n (-\al H_o)^i \xi \]
for all $n \geq 0$. The base case is trivial; assume the statement holds for any fixed $n \geq 0$. Then
\begin{align*}
\begin{pmatrix}
f_1^{n+1} \\ f_2^{n+1}
\end{pmatrix} &= -\al \xi -\al H_o \left( -\al \sum_{i=0}^{n} (-\al H_o)^i \xi \right) = -\al \xi -\al \sum_{i=1}^{n+1} (-\al H_o)^i \xi = -\al \sum_{i=0}^{n+1} (-\al H_o)^i \xi
\end{align*}
as required. Finally we conclude
\[ \text{LookAhead}n = \begin{pmatrix}
f_1^n \\ f_2^n
\end{pmatrix} = -\al \sum_{i=0}^n (-\al H_o)^i \xi = \text{CGD}n \]
as required. \hfill \qedsymbol

\section{Proof of Proposition~\ref{consistency_not_unique}}
\label{proof:cons_not_unqiue}

We prove that the two pairs of linear functions
\[ f_1 = f_2 = -2(x+y+1) \]
and
\[ f_1 = f_2 = -\frac{1}{2}(x+y-2) \]
are solutions to the consistency equations in the Tandem game with $\al = 1$. (See below for a generalization to any $\al > 0$.) For the first pair of functions, we have
\[ -\nabla_x \left( L^1(x, y+f_2) \right) = -\nabla_x \left( (x+y+2)^2 - 2x \right) = -2(x+y+1) = f_1 \]
for the first consistency equation and similarly
\[ -\nabla_y \left( L^2(x+f_1, y) \right) = -\nabla_x \left( (x+y+2)^2 - 2y \right) = -2(x+y+1) = f_2 \]
for the second. For the second pair of functions we similarly obtain
\[ -\nabla_x \left( L^1(x, y+f_2) \right) = -\nabla_x \left( \frac{1}{4}(x+y+2)^2 - 2x \right) = -\frac{1}{2}(x+y-2) = f_1 \]
for the first consistency equation and
\[ -\nabla_y \left( L^2(x+f_1, y) \right) = -\nabla_x \left( \frac{1}{4}(x+y+2)^2 - 2y \right) = -\frac{1}{2}(x+y-2) = f_2 \]
for the second. This shows that both functions are solutions to the consistency equations. For general $\al > 0$, we can similarly show that $f_1 = f_2 = ax+by+c$ with
\begin{align*}
a = \frac{\pm\sqrt{1+8\al}-1-4\al}{4\al} \quad ; \quad b = \frac{-2\al(1+a)}{1+2\al(1+a)} \quad ; \quad c= \frac{2\al}{1+2\al(1+a)}
\end{align*}
are two distinct solutions (depending on $\pm$) to the consistency equations in the Tandem game, noting that the denominators cannot be $0$ for $\al > 0$ (otherwise leading to a contradiction in the expression for $a$). This is left to the reader, noting that the proof for $\al = 1$ is sufficient to establish that consistent solutions are not always unique. \hfill \qedsymbol

\section{Proof of Proposition~\ref{consistency_sfps}}
\label{proof:cons_sfps}

Recall from the proof of Proposition \ref{consistency_not_unique} that the linear functions
\[ f_1 = f_2 = -2(x+y+1) \]
are consistent solutions to the Tandem game with $\al = 1$. The SFPs of the Tandem game are $(x,1-x)$ for each $x\in \R$, but none of these are preserved by the consistent solutions above since
\[ f_1(x,1-x) = f_2(x, 1-x) = -4 \neq 0 \,. \]
We conclude that consistency does not imply preservation of SFPs.

Moreover, we prove that any (non-zero) linear solution to the consistency equations cannot preserve more than one SFP in the Tandem game, for any opponent shaping rate \(\alpha\).
Assuming it did, we must have linear functions
\begin{align*}
f_1 &= ax+by+c \\
f_2 &= a'x+b'y+c'
\end{align*}
satisfying
\[ f_1(x, 1-x) = 0 = f_1(x', 1-x') \]
for some $x \neq x' \in \R$. Subtracting RHS from LHS we obtain $(x-x')(a-b) = 0$ hence $a = b$, which substituted again into the LHS yields $b = -c$. Applying the same method for $f_2$ we obtain $a' = b' = -c'$ and so $f_1,f_2$ take the form
\begin{align*}
f_1 &= a(x+y-1) \\
f_2 &= a'(x+y-1).
\end{align*}
Note that since \(f_1,f_2\) were assumed to be nonzero, it follows that $a, a' \neq 0$. Plugging these into the first consistency equation, we obtain
\[ a(x+y-1) = -2\al \left[ (1+a')\left((x+y)(1+a')-a'\right)-1\right] \,. \]
Comparing $x$ terms and constant terms yields
\[ a = -2\al(1+a')^2 \qquad \text{and} \qquad a = -2\al\left(1+a'+{a'}^2\right) \]
which concludes the contradiction $a' = 0$.\hfill \qedsymbol

\section{Proof of Proposition \ref{prop:hamiltonian_theory}}\label{appendix:hamiltonian_theory}

\paragraph{LOLA and SOS diverge.}
Assume $(x_0, y_0) \neq 0$ and $\al > 1$. We prove the more general claim that $p$-LOLA diverges for any $0 \leq p \leq 1$ (where $p$ may take a different value at each learning step), recalling that LOLA and SOS are both special cases of $p$-LOLA \citep{letcher_stable_2019}. Indeed, the $p$-LOLA gradient update is given by
\[ \begin{pmatrix}
h_1 \\ h_2
\end{pmatrix} = -\al(I-\al H_o)\xi + p \al^2 \lchi = -\al \begin{pmatrix}
y+\al x(1+p) \\ -x+\al y(1+p)
\end{pmatrix} \]
and we show that each update leads to increasing distance from the origin as follows:
\begin{align*}
\norm{(x+h_1, y+h_2)}^2 &= x^2-2x\al(y+\al x(1+p))+\al^2\left(y^2+\al^2 x^2(1+p)^2+2\al xy(1+p)\right) + \\
&\phantom{{}={}} y^2-2y\al(-x+\al y(1+p))+\al^2\left(x^2+\al^2 y^2(1+p)^2-2\al xy(1+p) \right) \\
&= \left(x^2+y^2\right)\left(1-\al^2(2p+1)+\al^4(1+p)^2\right) \\
&\geq \left(x^2+y^2\right) \left(1-\al^2+\al^4\right) \coloneqq \norm{(x,y)}^2\la
\end{align*}
where the inequality follows because the final expression in $p$ has positive derivative for $\al > 1$, hence minimized at $p=0$. Now $\la > 1$ for any $\al > 1$, so we conclude by induction that
\[ \norm{(x_n, y_n)}^2 \geq \la^n \norm{(x_0, y_0)}^2 \to \infty \]
as $n \to \infty$, provided $(x_0, y_0) \neq 0$, as required.

\paragraph{Consistent solution converges.} We begin by showing that the following linear functions satisfy the consistency equations for the Hamiltonian game:
\[ \begin{pmatrix}
f_1 \\ f_2 
\end{pmatrix} = \frac{-\al}{1+2\al^2} \begin{pmatrix} y+2\al x \\ -x+2\al y
\end{pmatrix} \,. \]
Indeed, the RHS of the first consistency equation is
\begin{align*}
-\al \nabla_{x} \left( x\left(y-\al\frac{-x+2\al y}{1+2\al^2} \right) \right) &= \frac{-\al}{1+2\al^2} \Big( y(1+2\al^2)-\al(-x+2\al y)+\al x \Big) \\
&= \frac{-\al}{1+2\al^2} \Big( y+2\al x \Big) = f_1
\end{align*}
and similarly for the second equation.

To prove uniqueness, assume there is a second pair of linear update functions \(\hat{f}_1,\hat{f}_2\) also satisfying consistency. Let  \(a,b,c\in\mathbb{R}\) such that \(\hat{f}_1(x,y)=ax+by+c\). Note that substituting the second equation into the first yields
\begin{align*} \hat{f}_1(x,y) &= -\al \nabla_x \left(L^1(x, y-\al \nabla_y \left(L^2(x+\hat{f}_1(x,y), y)\right))\right)\\
&= -\al \left(y+\al\left(2x+\hat{f}_1(x,y)+x\nabla_x\hat{f}_1(x,y)+y\nabla_y\hat{f}_1(x,y)+xy\nabla_{xy}\hat{f}_1(x,y)\right)\right) \end{align*}
Expanding the above and substituting the equation for \(\hat{f}_1\), we obtain
\[ ax+by+c = - 2\al^2 x(1 + a) -\al y(1+2\al b)  -\al^2 c \]
for all $x,y \in \R$, which yields (by comparing coefficients)
\[ a(1+2\al^2) = -2\al^2 \qquad \text{;} \qquad b(1+2\al^2) = -\al \qquad \text{;} \qquad c(1+\al^2) = 0.\]
It follows that
\[ \hat{f}_1(x,y) = \frac{-\al}{1+2\al^2} \Big( y+2\al x \Big) = f_1(x,y),\]
proving the uniqueness of \(f_1\). Since \(f_2\) is directly determined by \(f_1\) via the second consistency equation, this concludes the proof.

Finally we prove that this linear update leads to decreasing distance from the origin as follows:
\begin{align*}
\norm{(x+f_1, y+f_2)}^2 &= x^2-\frac{2x\al}{1+2\al^2}(y+2\al x)+\frac{\al^2}{(1+2\al^2)^2}\left(y^2+4 \al^2 x^2+4\al xy \right) + \\
&\phantom{{}={}} y^2-\frac{2y\al}{1+2\al^2}\al(-x+2\al y)+\frac{\al^2}{(1+2\al^2)^2}\left(x^2+4\al^2 y^2-4\al xy \right) \\
&= \left(x^2+y^2\right)\left(1-\frac{\al^2(3+4\al^2)}{(1+2\al^2)^2} \right) \coloneqq \norm{(x,y)}^2\la \,.
\end{align*}
Notice that the derivative of $\la$ is strictly negative in $\al$ while its limit as $\al \to \infty$ is $0$, with value $1$ at $\al = 0$, hence $\abs{\la} = \la < 1$ for any $\al > 0$. We conclude by induction that
\[ \norm{(x_n, y_n)}^2 = \la^n \norm{(x_0, y_0)}^2 \to 0 \]
as $n \to \infty$, with $\la$ decreasing (hence the speed of convergence increasing) as $\al$ increases. \hfill \qedsymbol

\section{Training Details COLA}
\label{appendix:cola_training}
All code was implemented using Python. The code relies on the PyTorch library for autodifferentiability \citep{pytorch2019}.
\subsection{Polynomial games}
For the polynomial games, COLA uses a neural network with 1 non-linear layer for both $h_{1}(\theta^{1}, \theta^{2})$ and $h_{2}(\theta^{1}, \theta^{2})$. The non-linearity is a ReLU function. The layer has 8 nodes. For training, we randomly sample pairs of parameters on a [-1, 1] parameter region. In general, the size of the region is a hyperparameter. We use a batch size of 8. We found that training is improved with a learning rate scheduler. For the learning rate scheduling we use a $\gamma$ of 0.9. We train the neural network for 120,000 steps. To compute the consistency loss we use the squared distance measure. The optimizer used is Adam \citep{kingma2017adam}.
\subsection{Non-polynomial games}
For the non-polynomial games, we deploy a neural network with 3 non-linear layers using Tanh activation functions. Each layer has 16 nodes. For this type of game, the parameter region is set to [-7, 7], because the parameters will be squished into probability space, allowing us to explore the full probability space. During training, we used a batch size of 64. The optimizer used is Adam \citep{kingma2017adam}.

\newpage

\section{Further experimental results}
\label{appendix-further-results}
\subsection{Tandem game}
\label{appendix:tandem}
In this section, we will provide more empirical results on the Tandem game. First, in Figure \ref{fig:tandem_full}, we display a look-ahead regime where SOS, LOLA and HOLA8 diverge in training, whereas COLA and CGD do not. Second, in Figure \ref{fig:tandem_grad_fields} we compare the gradient fields of HOLA4 and COLA at a low and high look-ahead rate of 0.1 and 1.0 respectively. The updates are shown on the parameter region of interesting $\Theta$ for the Tandem game, which is [-1,1]. Figure \ref{fig:tandem_grad_field1} and \ref{fig:tandem_grad_field3} show that the solutions are very similar at low look-ahead rates but become dissimilar at high look-ahead rates, shown in Figure \ref{fig:tandem_grad_field2} and \ref{fig:tandem_grad_field4}. This also confirms our quantitative observation in Table \ref{tab:tandem_cos}.

\begin{figure}[hbt!]
 \begin{subfigure}[]{0.49\linewidth}
  \centering
	\includegraphics[width=\linewidth]{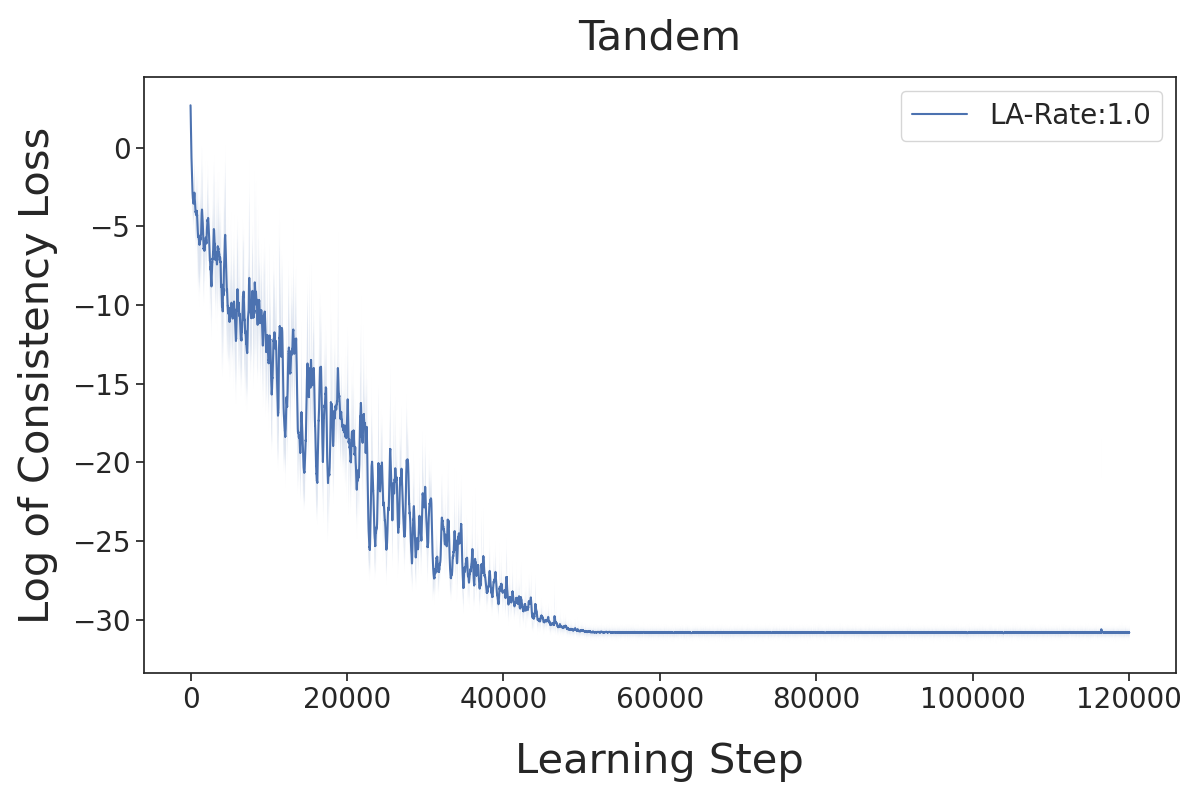}
	\caption{}
	\label{fig:tandem_cons_div}
 \end{subfigure}
\begin{subfigure}[]{0.49\linewidth}
 \centering
  	\includegraphics[width=\linewidth]{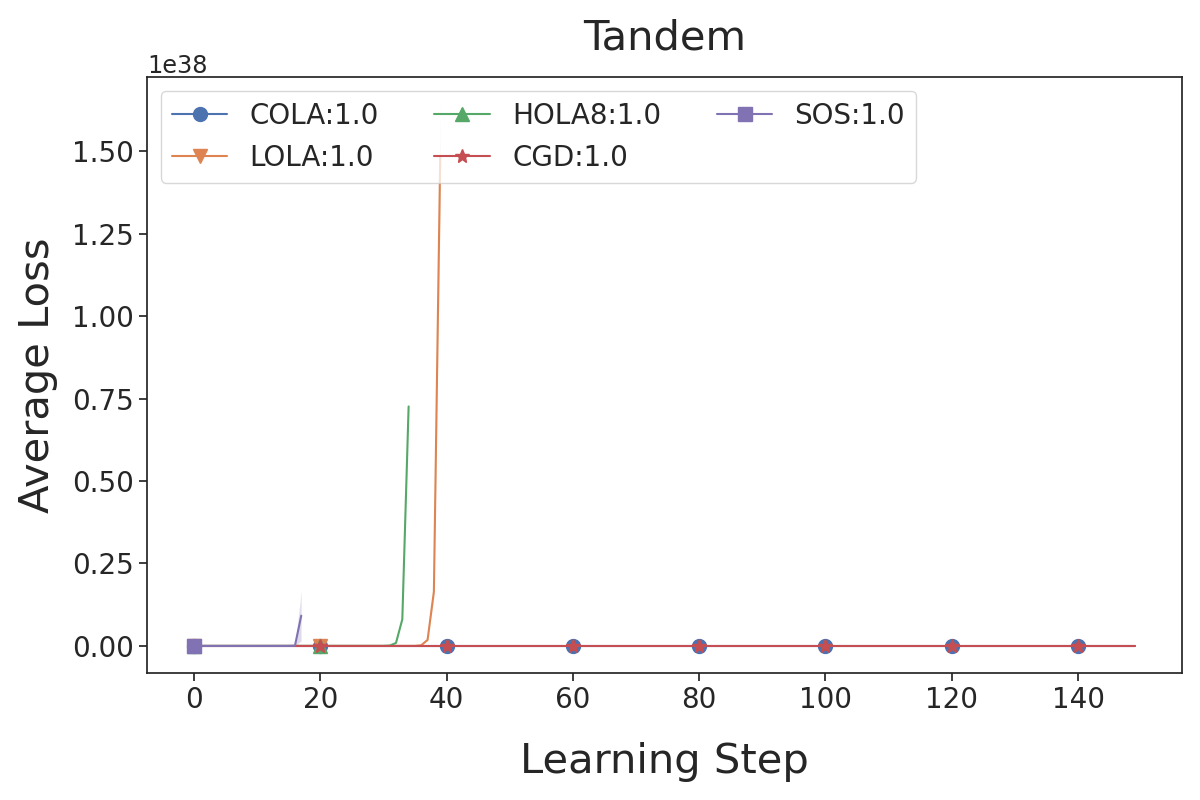}
  	  	\caption{}
  	\label{fig:tandem_play_div}
 \end{subfigure}
 \caption{(a): Consistency loss of COLA at a look-ahead rate of 1.0. (b): Solutions on the Tandem game by COLA, LOLA, HOLA8, CGD, and SOS with a look-ahead rate of 1.0. The standard deviation for the initialization of parameters used here is 0.1, which is standard in the literature \citep{letcher_stable_2019}.}
 \label{fig:tandem_full}
\end{figure}

\begin{figure}[hbt!]
\begin{subfigure}{.49\textwidth}
  \centering
  \includegraphics[width=\linewidth]{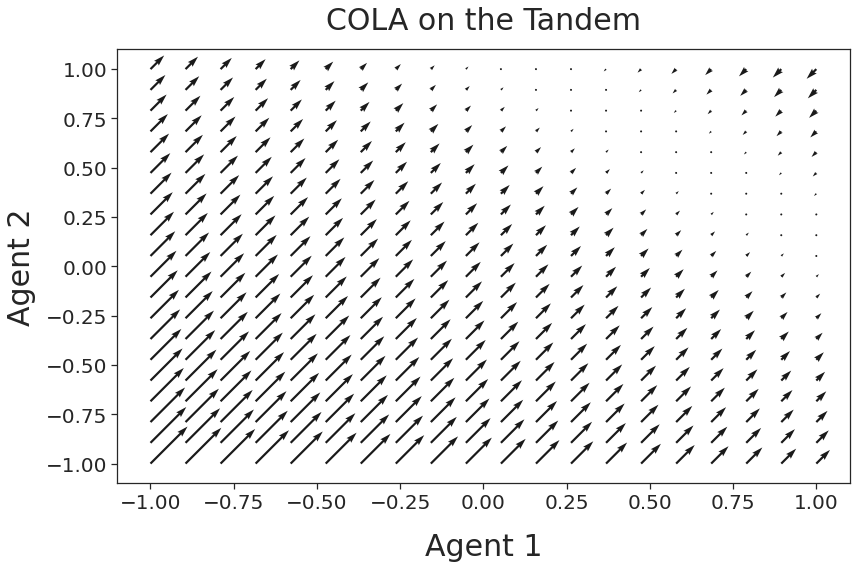}
  \caption{}
  \label{fig:tandem_grad_field1}
\end{subfigure}
\begin{subfigure}{.49\textwidth}
  \centering
  \includegraphics[width=\linewidth]{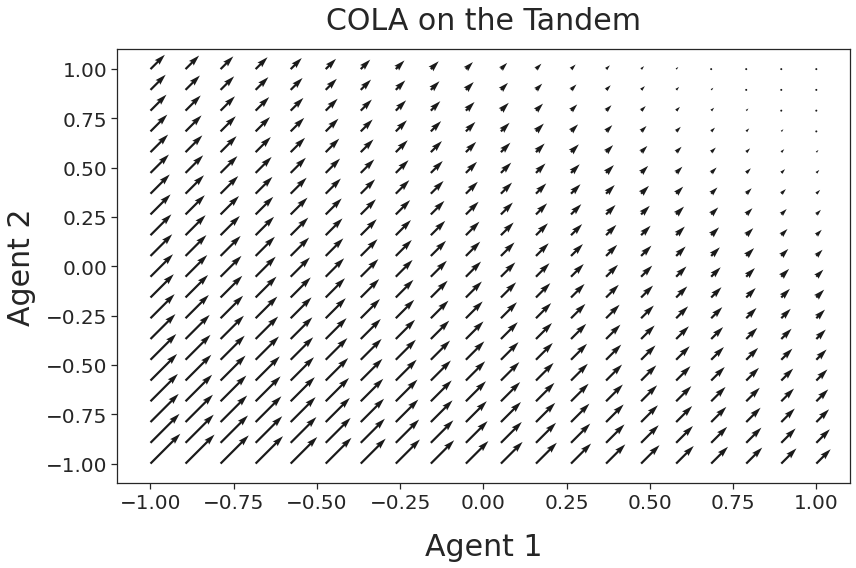}
  \caption{}
  \label{fig:tandem_grad_field2}
\end{subfigure}
\begin{subfigure}{.49\textwidth}
  \centering
  \includegraphics[width=\linewidth]{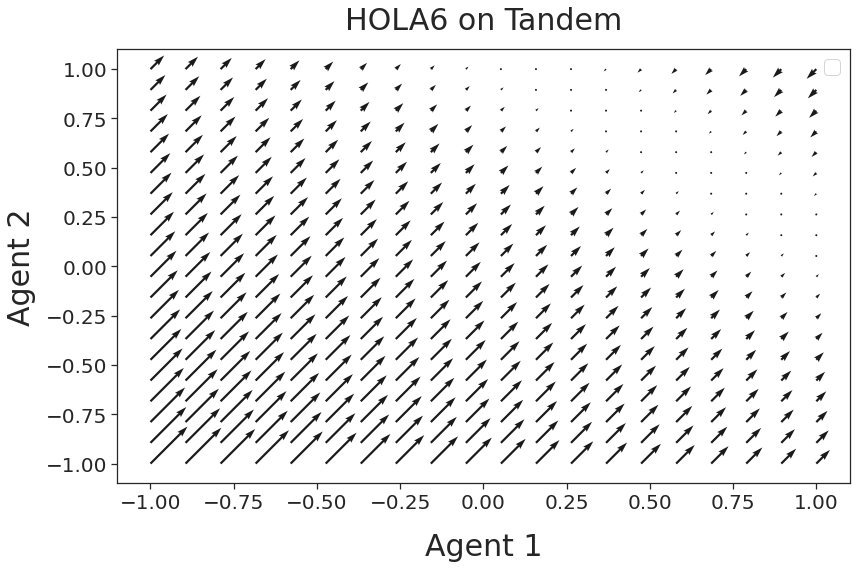}
  \caption{}
  \label{fig:tandem_grad_field3}
\end{subfigure}
\begin{subfigure}{.49\textwidth}
  \centering
  \includegraphics[width=\linewidth]{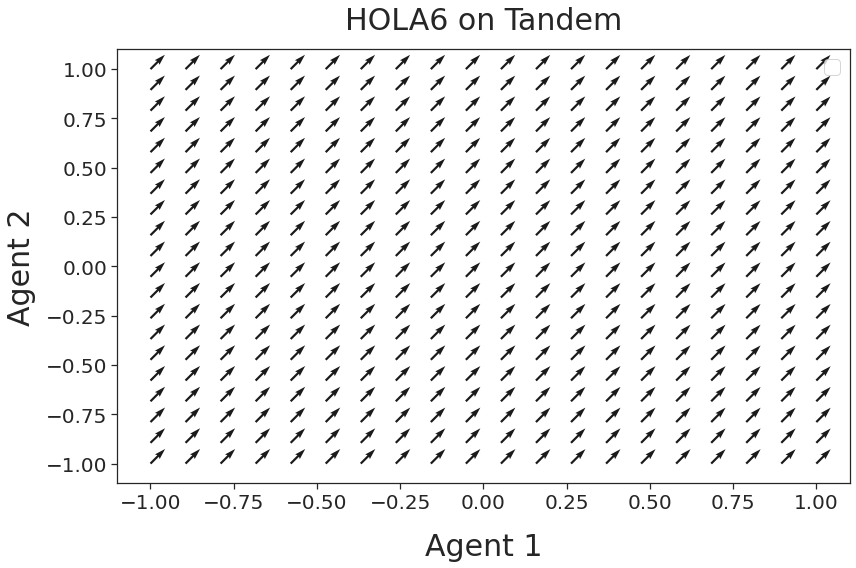}
  \caption{}
  \label{fig:tandem_grad_field4}
\end{subfigure}
\caption{Gradients field of the Tandem game at two different look-ahead rates, 0.1 and 1.0.}
\label{fig:tandem_grad_fields}
\end{figure}

\newpage
\subsection{Balduzzi and Hamiltonian game}
\label{appendix:hamiltonian}
The Hamiltonian game was originally introduced in \citep{balduzzi_mechanics_2018} as a minimal example of Hamiltonian dynamics. Recall that its loss function is
\begin{equation}
    L^{1}(x, y)=xy \quad \text{and} \quad L^{2}(x, y)=-xy
\end{equation}

The Balduzzi game was introduced to investigate the behaviour of differentiable game algorithms when a weak attractor is coupled with strong rotational forces in the Hamiltonian dynamics \citep{balduzzi_mechanics_2018}, captured by the losses
\begin{equation}
    L^{1}(x,y)=\frac{1}{2}x^2+10xy \quad \text{and} \quad L^{2}(x,y)=\frac{1}{2}y^2-10xy
\end{equation}

Results for both games are displayed in Figures \ref{fig:full_bal} and \ref{fig:full_ham}. COLA at high look-ahead rates converges considerably faster than the other methods on both games. This supports Proposition \ref{prop:hamiltonian_theory} empirically. In Table \ref{tab:ham_cons}, \ref{tab:ham_sim} and \ref{tab:bal_cons} we find similar consistency loss behaviour as we do for Tandem. Note however, that the look-ahead thresholds are at significantly different levels. On the Balduzzi game, the threshold is much lower around 0.02, whereas for the Hamiltonian game it is between 0.1 and 0.4.

\begin{figure*}[hbt!]
 \centering
 \begin{subfigure}[]{0.48\linewidth}
	\includegraphics[width=\linewidth]{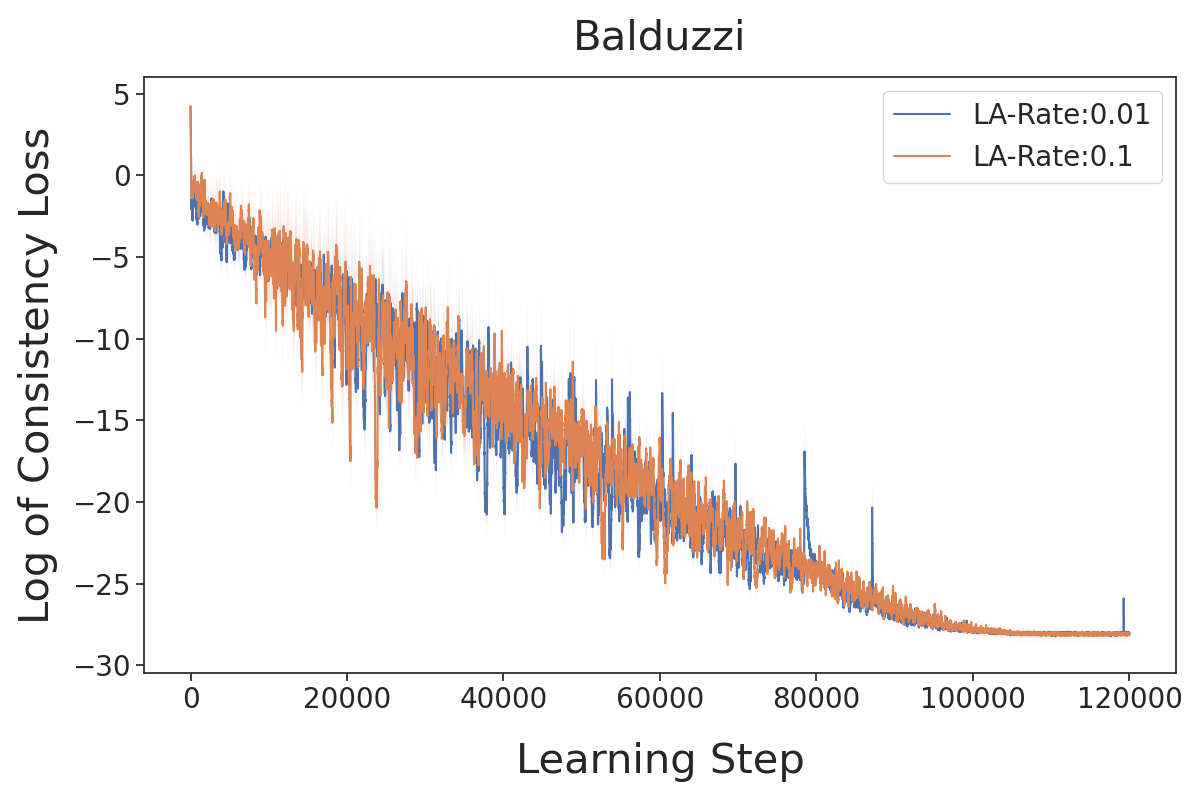}
	\caption{}
	\label{fig:bal_cons}
 \end{subfigure}
\begin{subfigure}[]{0.48\linewidth}
  	\includegraphics[width=\linewidth]{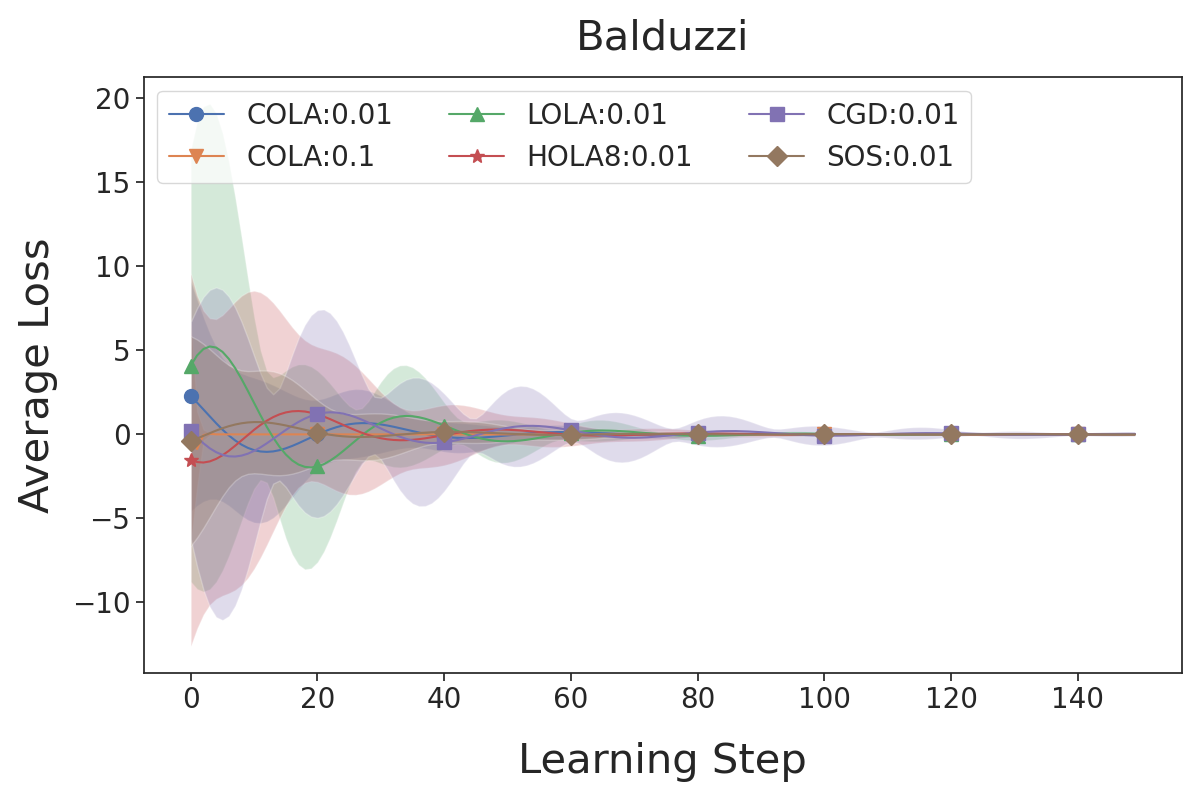}
  	  	\caption{}
  	\label{fig:bal_play}
 \end{subfigure}
 \caption{(a): Consistency losses of COLA at different look-ahead rates. (b): Solutions on the Balduzzi game by COLA, LOLA, HOLA8, CGD, and SOS with a look-ahead rate of 0.01 (and 0.1 additionally for COLA). The standard deviation for the initialization of parameters used here is 1.0}
 \label{fig:full_bal}
\end{figure*}

\begin{figure*}[hbt!]
 \centering
 \begin{subfigure}[]{0.48\linewidth}
	\includegraphics[width=\linewidth]{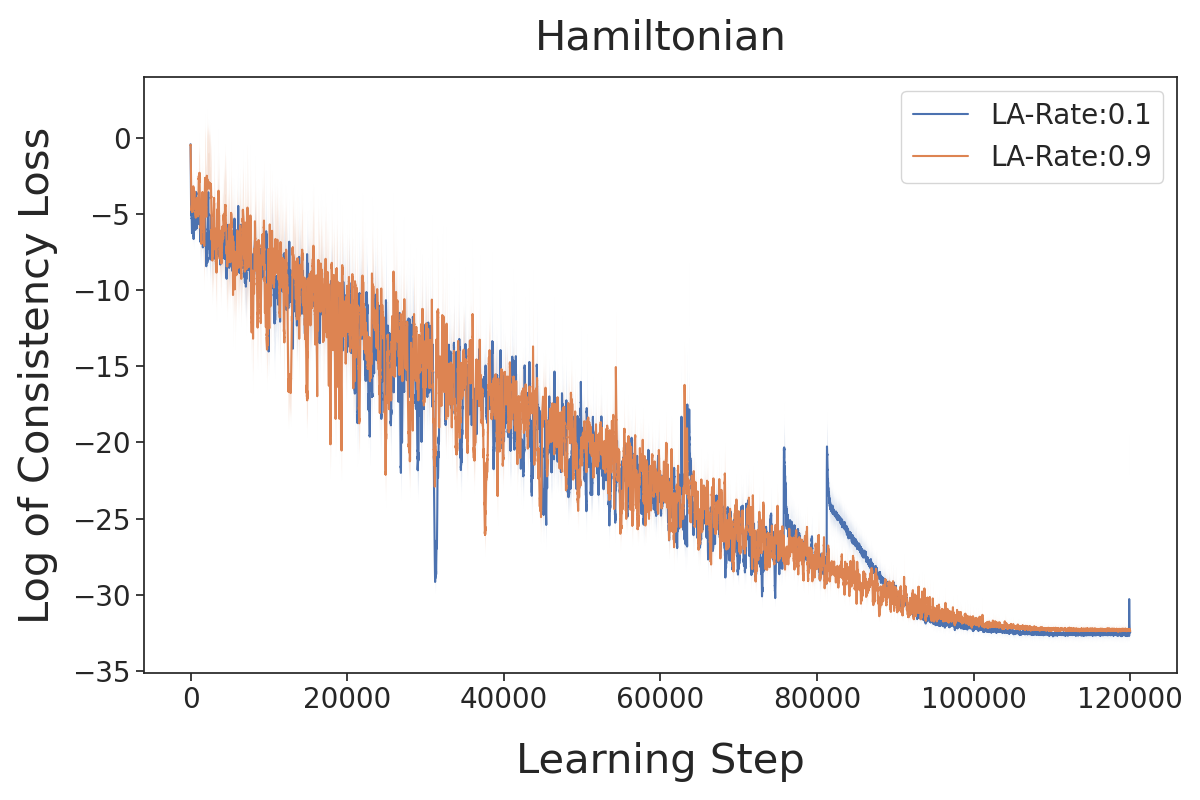}
	\caption{}
	\label{fig:ham_cons}
 \end{subfigure}
\begin{subfigure}[]{0.48\linewidth}
  	\includegraphics[width=\linewidth]{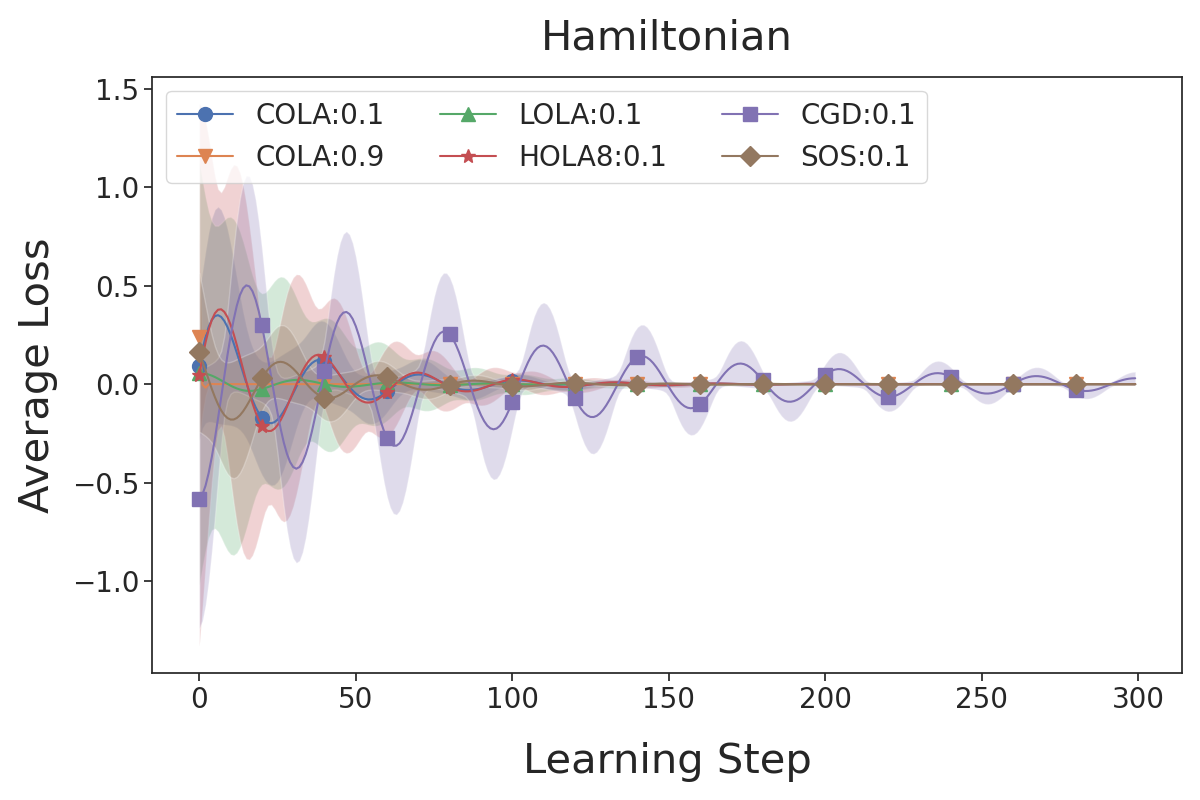}
  	  	\caption{}
  	\label{fig:ham_play}
 \end{subfigure}
 \caption{(a): Consistency losses of COLA at different look-ahead rates. (b): Solutions on the Hamiltonian game by COLA, LOLA, HOLA8, CGD, and SOS with a look-ahead rate of 0.1 (and 0.9 additionally for COLA). The standard deviation for the initialization of parameters used here is 1.0}
 \label{fig:full_ham}
\end{figure*}
\newpage

\begin{table}[hbt!]
\caption{On the Hamiltonian game: (a) Log of the squared consistency loss. (b) Cosine similarity between COLA and LOLA, HOLA3, and HOLA6 over different look-ahead rates. The values represent the mean of a 1,000 samples, uniformly sampled from the parameter space $\Theta$. The error bars represent one standard deviation and capture the variance over 10 different COLA training runs.}
\centering
\subfloat[]{
\begin{tabular}{l|l l l l}
\label{tab:ham_cons}$\alpha$ & LOLA    & HOLA3    & HOLA6    & \multicolumn{1}{c}{COLA} \\ \hline
0.9      & 6.99    & 4.63     & 13.68    & 1e-14$\pm$2e-15              \\ \hline
0.5      & 12.67   & 13.77    & 13.00    & 1e-14$\pm$2e-15              \\ \hline
0.4      & 5.78    & 7.14     & 6.38     & 1e-14$\pm$2e-15              \\ \hline
0.1      & 0.08    & 0.01     & 2e-6  & 7e-15$\pm$1e-15              \\ \hline
0.05     & 2e-5 & 4e-10 & 3e-15 & 1e-14$\pm$2e-14              \\ \hline
\end{tabular}}
\quad

\subfloat[]{\begin{tabular}{l|l l l}
\label{tab:ham_sim}$\alpha$ & LOLA & HOLA3 & HOLA6  \\ \hline
0.9   & 1.00$\pm$0.00  & -1.00$\pm$0.00  & 0.50$\pm$0.00 \\ \hline
0.5   & 1.00$\pm$0.00  & 1.00$\pm$0.00   & 1.00$\pm$0.00  \\ \hline
0.4   & 1.00$\pm$0.00  & 1.00$\pm$0.00   & 1.00$\pm$0.00  \\ \hline
0.1   & 1.00$\pm$0.00  & 1.00$\pm$0.00   & 1.00$\pm$0.00    \\ \hline
0.05  & 1.00$\pm$0.00  & 1.00$\pm$0.00   & 1.00$\pm$0.00    \\ \hline
\end{tabular}}
\end{table}

\begin{table}[hbt!]
\centering
\caption{On the Balduzzi game: (a) Log of the squared consistency loss. (b) Cosine similarity between COLA and LOLA, HOLA3, and HOLA6 over different look-ahead rates. The values represent the mean of a 1,000 samples, uniformly sampled from the parameter space $\Theta$. The error bars represent one standard deviation and capture the variance over 10 different COLA training runs.}

\subfloat[]{
\begin{tabular}{l|l l l l}
\label{tab:bal_cons}$\alpha$ & LOLA    & HOLA3    & HOLA6    & \multicolumn{1}{c}{COLA} \\ \hline
0.9      & 2e+6 & 5e+10 & 4e+17 & 5e-13$\pm$1e-13              \\ \hline
0.1      & 3e+2 & 1.e+3  & 2e+4  & 7e-13$\pm$1e-13              \\ \hline
0.05     & 2e+1 & 4.01     & 1.03     & 7e-13$\pm$9e-14              \\ \hline
0.03     & 2.16    & 0.07     & 0.01     & 8e-13$\pm$1e-13              \\ \hline
0.01     & 0.03    & 1e-5  & 2e-10 & 7e-13$\pm$1e-13              \\ \hline
\end{tabular}}
\quad
\subfloat[]{
\begin{tabular}{l|l l l}
\label{tab:bal_sim}$\alpha$ & LOLA          & HOLA3          & \multicolumn{1}{c}{HOLA6}         \\ \hline
0.9      & 1.00$\pm$0.00 & -1.00$\pm$0.00 & 0.01$\pm$4e-9 \\ \hline
0.1      & 1.00$\pm$0.00 & 1.00$\pm$0.00  & 0.40$\pm$1e-8 \\ \hline
0.05     & 1.00$\pm$0.00 & 1.00$\pm$0.00  & 1.00$\pm$0.00 \\ \hline
0.03     & 1.00$\pm$0.00 & 1.00$\pm$0.00  & 1.00$\pm$0.00 \\ \hline
0.01     & 1.00$\pm$0.00 & 1.00$\pm$0.00  & 1.00$\pm$0.00 \\ \hline
\end{tabular}}
\end{table}

\newpage

\subsection{Matching Pennies}
\label{appendix:mp}
Matching Pennies (MP) is a single-shot, zero-sum game, where two players, A and B, each flip a biased coin \citep{lee_application_1967}. Player A wins if the outcomes of both flips are the same and player B wins if they are different. 

\begin{table}[hbt!]
\centering
\caption{Payoff Matrix for the Matching Pennies game.}
\begin{tabular}{l|l l}\label{tab:imp}
     & Head     & Tail     \\ \hline
Head & (+1, -1) & (-1, +1) \\ \hline
Tail & (-1, +1) & (+1, -1) \\ \hline
\end{tabular}
\end{table}

\begin{table}[hbt!]
\centering
\caption{Cosine similarities over multiple COLA training runs on the MP and Tandem game for different look-ahead rates.}
\label{tab:self_sim}
\begin{tabular}{l|l}
Game@LR    & Cosine Sim      \\ \hline
MP@10     & 0.97 $\pm$ 0.01 \\ \hline
MP@0.5    & 0.99 $\pm$ 0.01   \\ \hline
Tandem@0.1 & 1.00 $\pm$ 0.00           \\ \hline
Tandem@1.0 & 0.98 $\pm$ 0.01    \\ \hline
\end{tabular}
\end{table}

The visual difference between the updates found by COLA and HOLA4 is shown in Figure \ref{fig:imp_grad_fields}. Whereas they are very similar at a relatively low look-ahead rate, at a high look-ahead rate HOLA4's gradient field shows signs of high variance, especially around the origin, whereas COLA's gradient field appears to learn a robust update function. Note that HOLA4's update function results in the learning behaviour shown in Figure \ref{fig:mp_div}.

\begin{figure}[hbt!]
\begin{subfigure}{.49\textwidth}
  \centering
  \includegraphics[width=.99\linewidth]{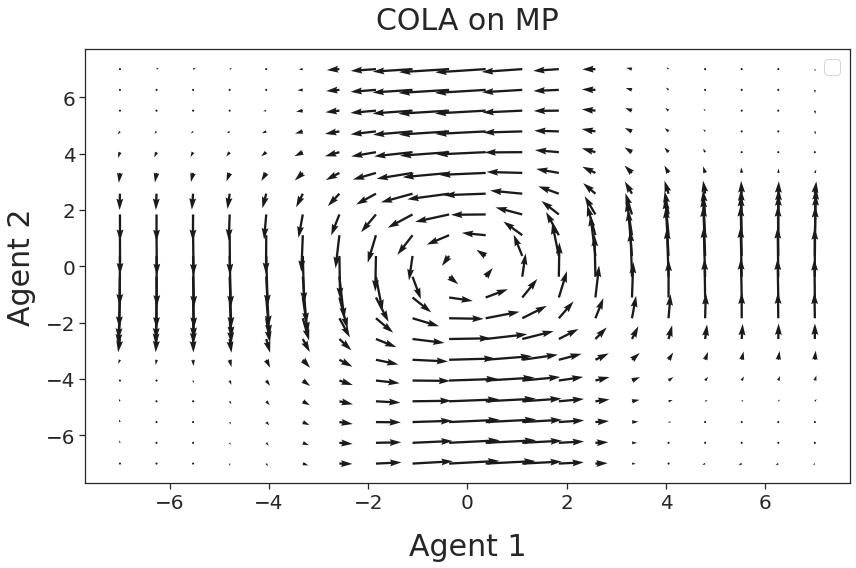}
  \caption{}
  \label{fig:imp_grad_field1}
\end{subfigure}
\begin{subfigure}{.49\textwidth}
  \centering
  \includegraphics[width=.99\linewidth]{imp/grad_field_cola_10.png}
  \caption{}
  \label{fig:imp_grad_field5}
\end{subfigure}
\begin{subfigure}{.49\textwidth}
  \centering
  \includegraphics[width=.99\linewidth]{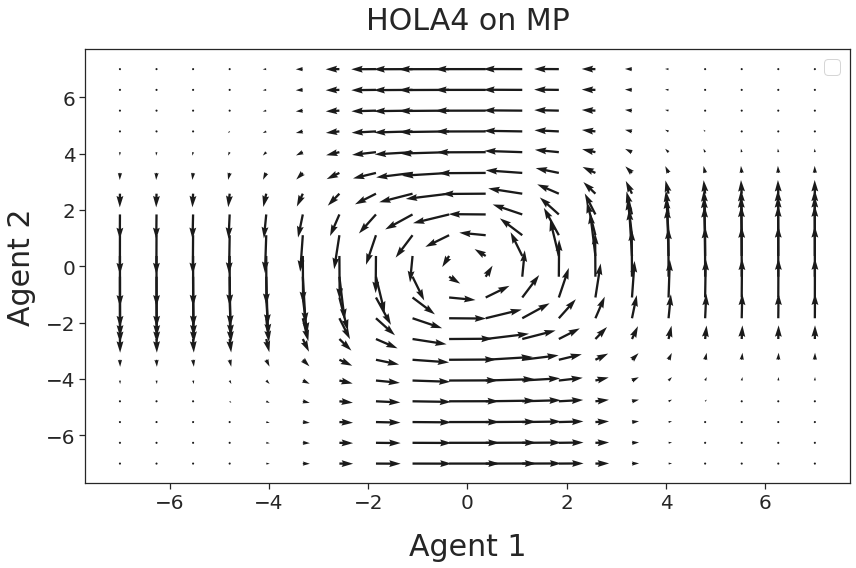}
  \caption{}
  \label{fig:imp_grad_field3}
\end{subfigure}
\begin{subfigure}{.49\textwidth}
  \centering
  \includegraphics[width=.99\linewidth]{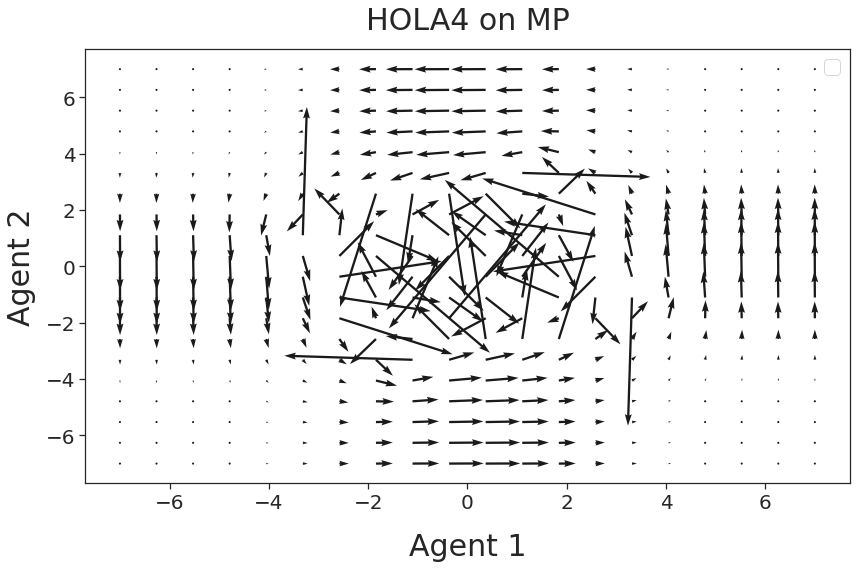}
  \caption{}
  \label{fig:imp_grad_field4}
\end{subfigure}
\caption{Gradients field of the MP game at two different look-ahead rates: 0.5 (RHS) and 10 (LHS). COLA is on the upper row, HOLA4 is on the lower row.}
\label{fig:imp_grad_fields}
\end{figure}
\newpage

\subsection{Ultimatum game}
\label{appendix:ult}
Here we provide more empirical results on the Ultimatum game. First, in Table \ref{tab:ult_cons} and \ref{tab:ult_sim} we display the consistency losses of COLA and HOLA at different look-ahead rates. In comparison to the polynomial games, here we observe that COLA's consistency losses are not as low as HOLA\textit{n}'s losses at low look-ahead rates. Nonetheless, they are low enought to constitute a consistent solution. Moreover, the COLA and HOLA\textit{n} solutions are very similar according to the cosine similarity score. Qualitatively, we compare the updates in Figure \ref{fig:ultimatum_grad_fields}. Similarly to the MP game, we observe that HOLA4's update shows higher variance around the origin that COLA's update. This variance is reflected in HOLA4's learning behaviour shown in Figure \ref{fig:ult_play_high}, where HOLA4 do not converge to the fair solution consistently.

\begin{table}[hbt!]
\caption{On the Ultimatum game: Over multiple look-ahead rates we compare (a) the consistency losses and (b) the cosine similarity between COLA and LOLA, HOLA2, and HOLA4. The values represent the mean of a 1,000 samples, uniformly sampled from the parameter space $\Theta$. The error bars represent one standard deviation and capture the variance over 10 different COLA training runs.}
\centering
\subfloat[]{
\begin{tabular}{l|l l l l}
\label{tab:ult_cons}$\alpha$ & LOLA     & HOLA2    & HOLA4    & \multicolumn{1}{c}{COLA} \\ \hline
1.1      & 2e-3  & 5e-3  & 0.01     & 4e-4$\pm$5e-5                \\ \hline
0.7      & 4e-4  & 2e-4  & 2e-4  & 7e-5$\pm$1e-5                \\ \hline
0.3      & 3e-5  & 2e-7  & 5e-8  & 3e-6$\pm$2e-6                \\ \hline
0.1      & 2e-7  & 1e-11 & 6e-13 & 4e-6$\pm$4e-6                \\ \hline
0.001    & 3e-15 & 4e-17 & 9e-17 & 4e-6$\pm$3e-6                \\ \hline
\end{tabular}}
\quad%
\subfloat[]{\begin{tabular}{l|l l l}
\label{tab:ult_sim}
$\alpha$ & \multicolumn{1}{c}{LOLA} & \multicolumn{1}{c}{HOLA2} & \multicolumn{1}{c}{HOLA4} \\ \hline
1.1      & 0.96$\pm$0.01                 & 0.96$\pm$0.01                  & 0.95$\pm$0.01                 \\ \hline
0.7      & 0.98$\pm$0.01                 & 0.98$\pm$0.01                  & 0.98$\pm$0.01                 \\ \hline
0.3      & 0.99$\pm$0.01                 & 0.99$\pm$0.01                  & 0.99$\pm$0.01                 \\ \hline
0.1      & 0.99$\pm$0.01                 & 0.99$\pm$0.01                  & 0.99$\pm$0.01                 \\ \hline
0.001    & 0.99$\pm$0.01                 & 0.99$\pm$0.01                  & 0.99$\pm$0.01                 \\ \hline
\end{tabular}}
\end{table}

\begin{figure}[hbt!]
\begin{subfigure}{.49\textwidth}
  \centering
  \includegraphics[width=.99\linewidth]{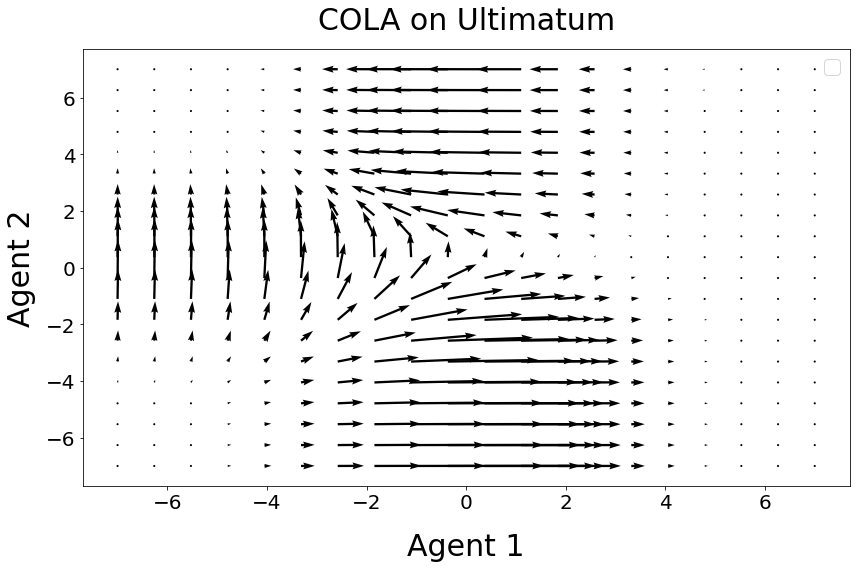}
  \label{fig:ultimatum_grad_field1}
\end{subfigure}
\begin{subfigure}{.49\textwidth}
  \centering
  \includegraphics[width=.99\linewidth]{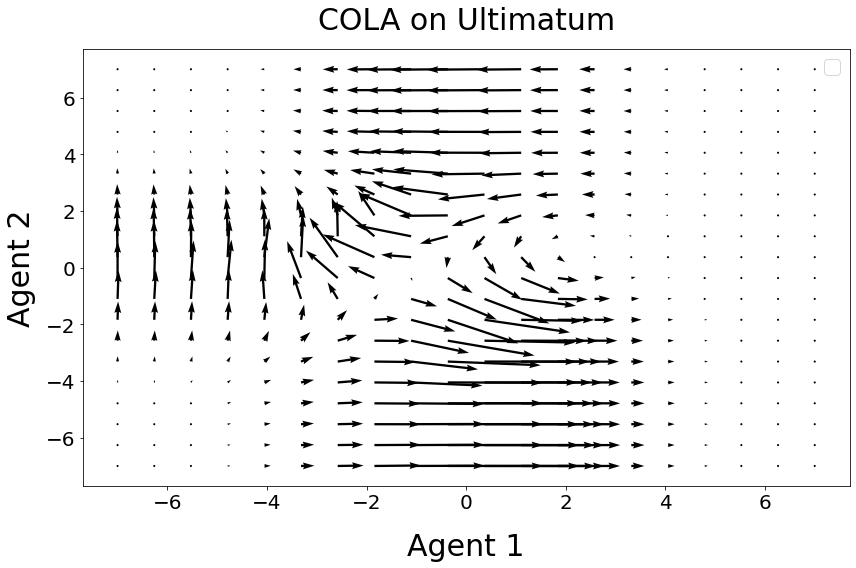}
  \label{fig:ultimatum_grad_field2}
\end{subfigure}
\begin{subfigure}{.49\textwidth}
  \centering
  \includegraphics[width=.99\linewidth]{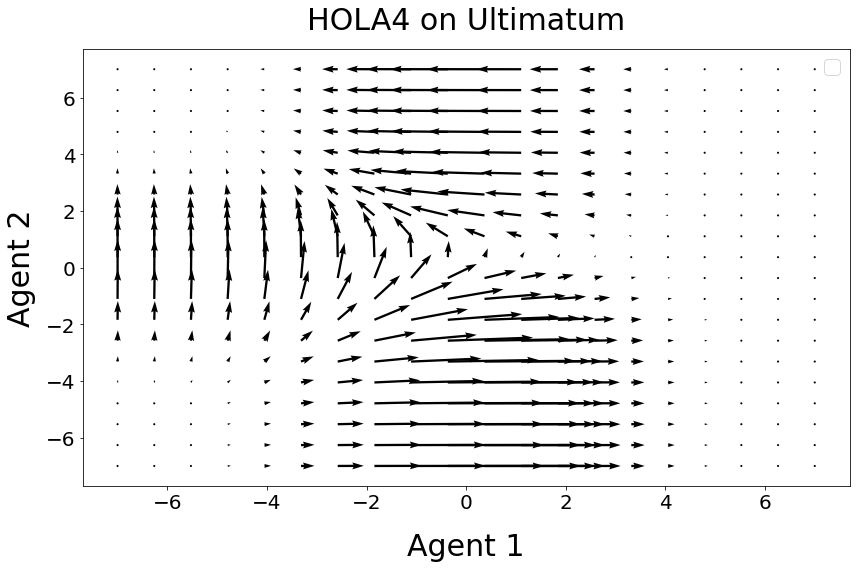}
  \label{fig:ultimatum_grad_field3}
\end{subfigure}
\begin{subfigure}{.49\textwidth}
  \centering
  \includegraphics[width=.99\linewidth]{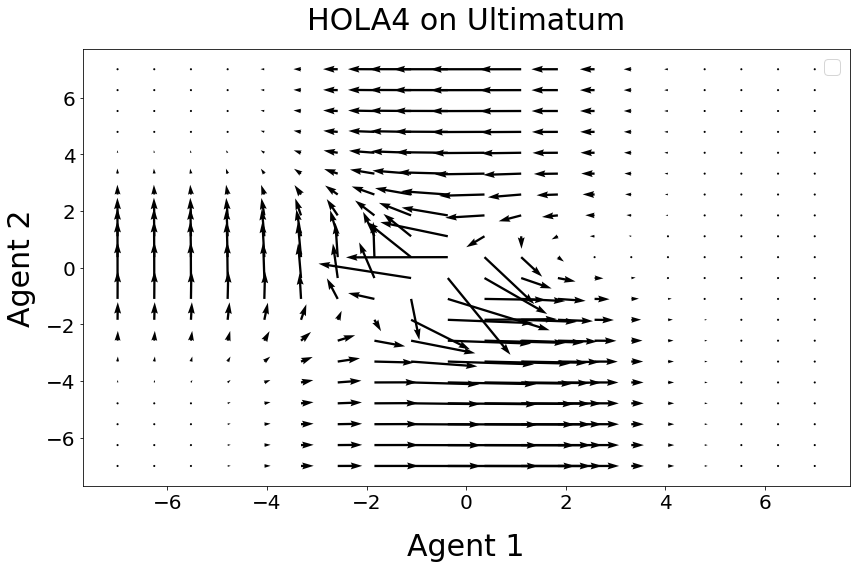}
  \label{fig:ultimatum_grad_field4}
\end{subfigure}
\caption{Gradients field of the ultimatum game at two different look-ahead rates: 0.2 (RHS) and 1.1 (LHS). COLA is on the upper row, HOLA4 is on the lower row.}
\label{fig:ultimatum_grad_fields}
\end{figure}

\begin{figure}[hbt!]
 \centering
\begin{subfigure}[]{0.49\linewidth}
  	\includegraphics[width=\linewidth]{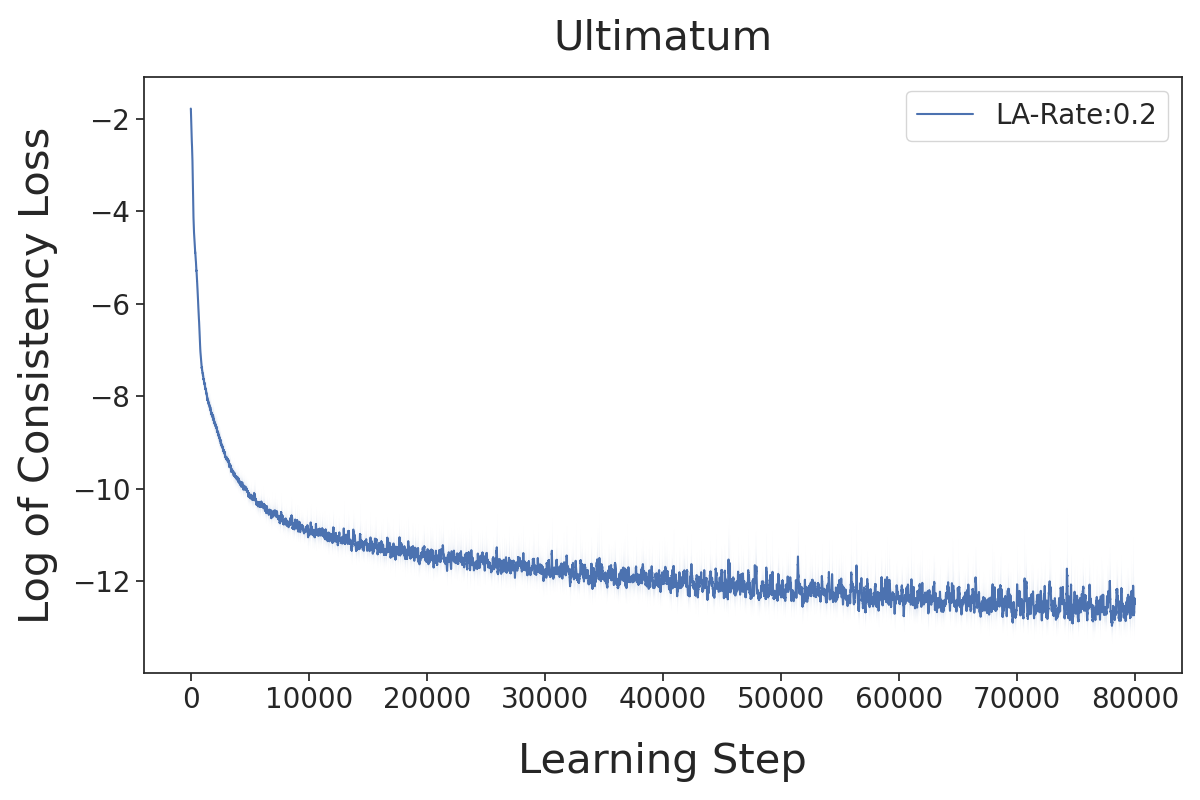}
  	  	\caption{}
  	\label{fig:ult_cons_low}
 \end{subfigure}
 \begin{subfigure}[]{0.49\linewidth}
	\includegraphics[width=\linewidth]{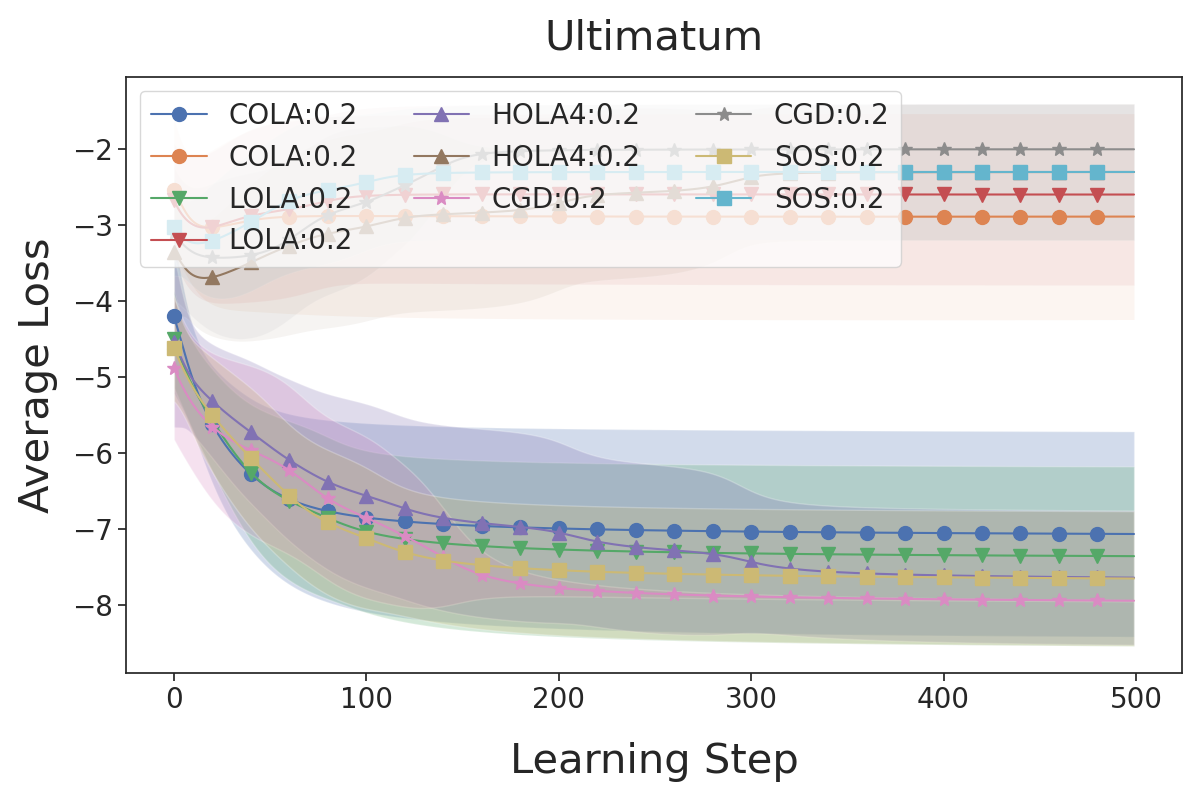}
	\caption{}
	\label{fig:ult_play_low}
 \end{subfigure}
 \begin{subfigure}[]{0.49\linewidth}
  	\includegraphics[width=\linewidth]{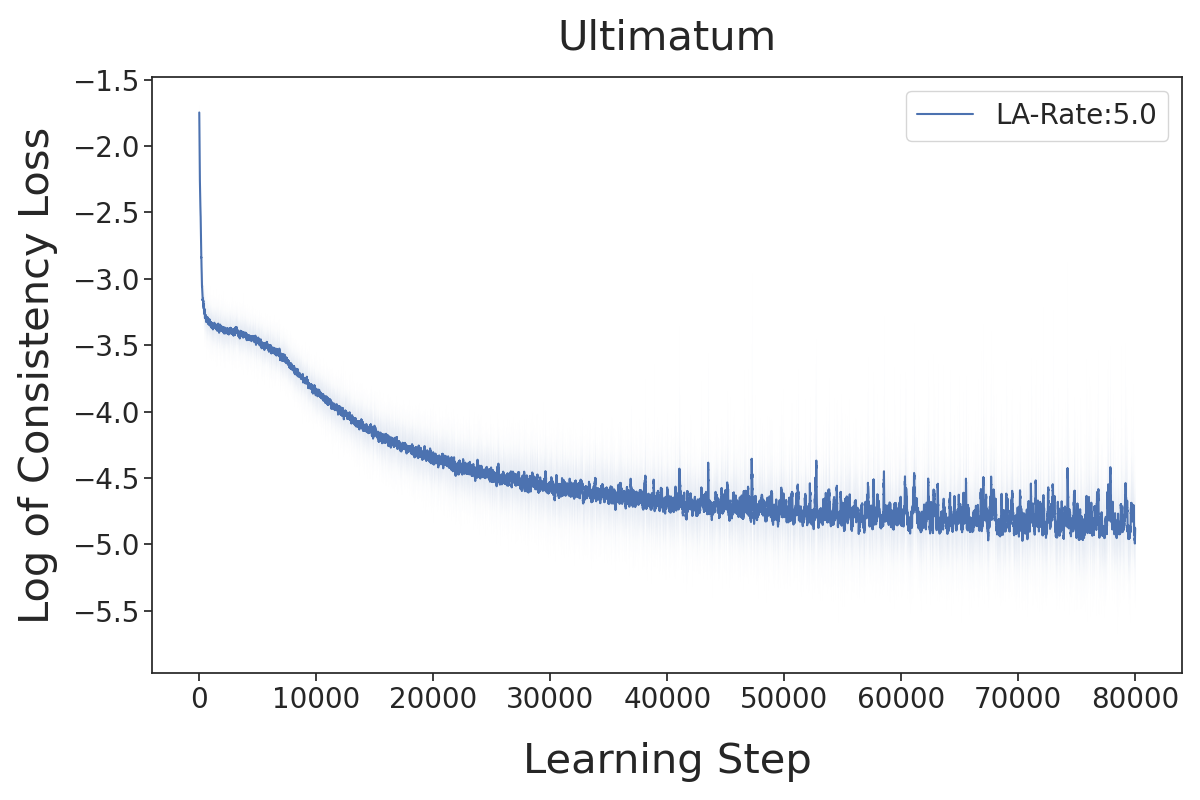}
  	  	\caption{}
  	\label{fig:ult_cons_high}
 \end{subfigure}
 \begin{subfigure}[]{0.49\linewidth}
  	\includegraphics[width=\linewidth]{ultimatum/learning_outcome_Ultimatum_5_0.png}
  	  	\caption{}
  	\label{fig:ult_play_high}
 \end{subfigure}
 \caption{(a) and (c): Consistency losses of COLA at different look-ahead rates. (b) and (d): Solutions on the Ultimatum game by COLA, LOLA, HOLA4, CGD, and SOS with a look-ahead rate of 0.2 and 5.0. The standard deviation for the initialization of parameters used here is 1.0.}
 \label{fig:ultimatum_appendix}
\end{figure}

\newpage

\subsection{Chicken game}
\label{appendix:chicken_game}

In the Chicken game, an agent can either choose to yield to avoid a catastrophic payoff but face a small punishment if they are the only agent to yield. Imagine a game where two agents drive towards each other in their cars. If both never swerve, they frontally crash into each other, an obviously catastrophic outcome. If any of the agents "chicken out", e.g. swerve, they do not crash but receive a small punishment for having chickened out. At the same time, the other agent is being rewarded for staying on track, as quantified in Table \ref{tab:chicken_payoff}.

\begin{table}[hbt!]
\centering
\caption{Payoff Matrix for the Chicken game.}
\begin{tabular}{l|c|c}\label{tab:chicken_payoff}
             & \multicolumn{1}{l|}{C (swerve)} & \multicolumn{1}{l}{D (straight)} \\ \hline
C (swerve)   & 0, 0                            & -1, +1                           \\ \hline
D (straight) & +1 ,-1                          & -100, -100                       \\ \hline
\end{tabular}
\end{table}

Next we report the consistency losses on the Chicken game in Table \ref{tab:chicken_cons}. We identify a look-ahead rate threshold where HOLA\textit{n}'s consistency loss becomes increasingly bigger with increasing order. For the Chicken game, this threshold is fairly low, between 0.01 and 0.05. Interestingly, we note that around a look-ahead rate of 0.05 and 0.1 it becomes harder to find a consistent solution for COLA.
\begin{table}[hbt!]
\centering
\caption{On the Chicken game: Over multiple look-ahead rates we compare (a) the consistency losses and (b) the cosine similarity between COLA and LOLA, HOLA2, and HOLA4. The values represent the mean of a 1,000 samples, uniformly sampled from the parameter space $\Theta$. The error bars represent one standard deviation and capture the variance over 10 different COLA training runs.}

\subfloat[]{
\begin{tabular}{l|l l l l l l}
\label{tab:chicken_cons}$\alpha$ & LOLA    & HOLA2   & HOLA4   & SOS     & CGD  & \multicolumn{1}{c}{COLA} \\ \hline
1.0      & 2429    & 3892    & 46637   & 1494    & 1677 & 0.01$\pm$0.01                \\ \hline
0.5      & 643     & 484     & 4320    & 475     & 2330 & 0.03$\pm$0.01                \\ \hline
0.1      & 11.99   & 7.69    & 73.28   & 2.73    & 8.46 & 0.70$\pm$0.10                \\ \hline
0.05     & 0.84    & 0.17    & 0.47    & 0.37    & 1.31 & 0.06$\pm$0.01                \\ \hline
0.01     & 9e-4 & 3e-6 & 6e-9 & 2e-4 & 0.04 & 5e-4$\pm$3e-4                \\ \hline
\end{tabular}}
\quad
\subfloat[]{
\begin{tabular}{l|l l l}
\label{tab:chicken_sim}$\alpha$ & \multicolumn{1}{c}{LOLA} & \multicolumn{1}{c}{HOLA2} & \multicolumn{1}{c}{HOLA4} \\ \hline
1.0                                           & 0.39$\pm$0.03             & 0.57$\pm$0.02              & 0.57$\pm$0.03             \\ \hline
0.5                                           & 0.54$\pm$0.04             & 0.62$\pm$0.03              & 0.64$\pm$0.04             \\ \hline
0.1                                           & 0.88$\pm$0.03             & 0.88$\pm$0.03              & 0.86$\pm$0.03             \\ \hline
0.05                                          & 0.93$\pm$0.03             & 0.93$\pm$0.03              & 0.93$\pm$0.03             \\ \hline
0.01                                          & 0.96$\pm$0.02             & 0.96$\pm$0.02              & 0.96$\pm$0.01             \\ \hline
\end{tabular}}
\end{table}

We also perform a qualitative comparison on the gradient fields of COLA and HOLA4 on the Chicken game (see Figure \ref{fig:chicken_grad_fields}). The difference at high look-ahead rates is pronounced, as HOLA4 shows high variance around the origin. Nonetheless, the gradient field leads to swerving in the actual game (see Figure \ref{fig:chicken_play_high}, which is a preferable outcome. Moreover, HOLA4 appears to converge to Swerving consistently despite the chaotic gradient field, showcasing that the behaviour in the game is not necessarily correlated with the qualitative analysis of the gradient fields.
\begin{figure}[hbt!]
\begin{subfigure}{.49\textwidth}
  \centering
  \includegraphics[width=.99\linewidth]{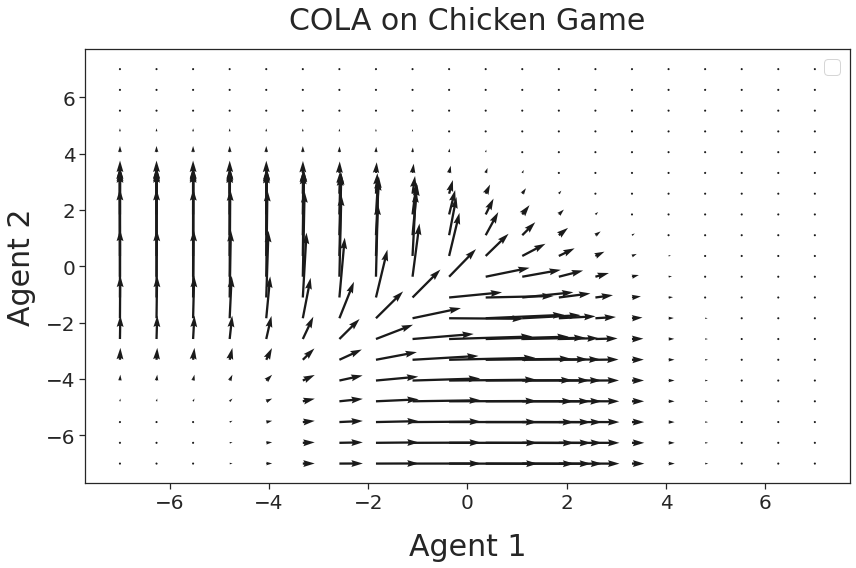}
  \caption{}
  \label{fig:chicken_grad_field1}
\end{subfigure}
\begin{subfigure}{.49\textwidth}
  \centering
  \includegraphics[width=.99\linewidth]{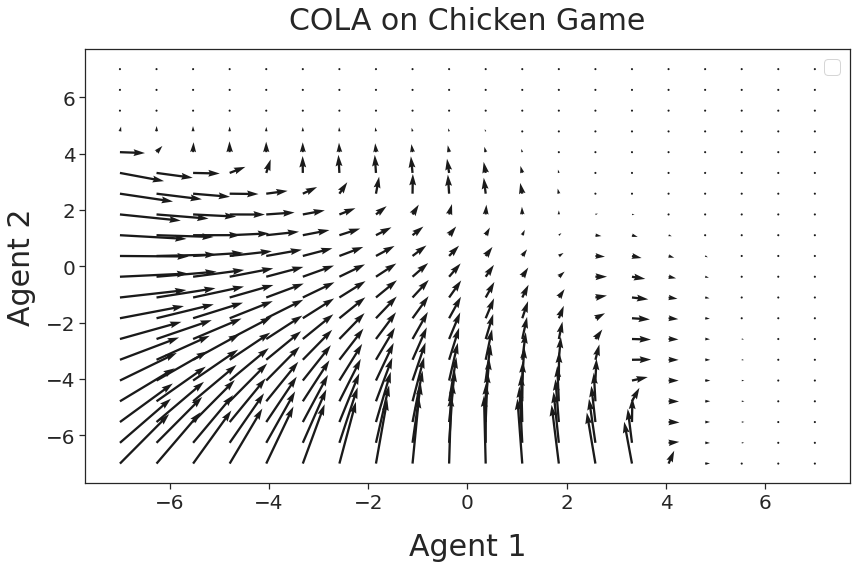}
  \caption{}
  \label{fig:chicken_grad_field2}
\end{subfigure}
\begin{subfigure}{.49\textwidth}
  \centering
  \includegraphics[width=.99\linewidth]{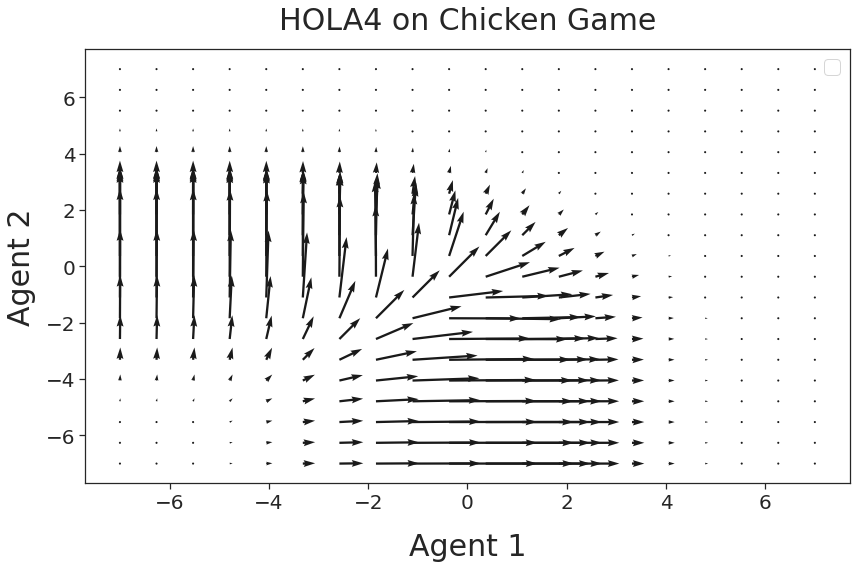}
  \caption{}
  \label{fig:chicken_grad_field3}
\end{subfigure}
\begin{subfigure}{.49\textwidth}
  \centering
  \includegraphics[width=.99\linewidth]{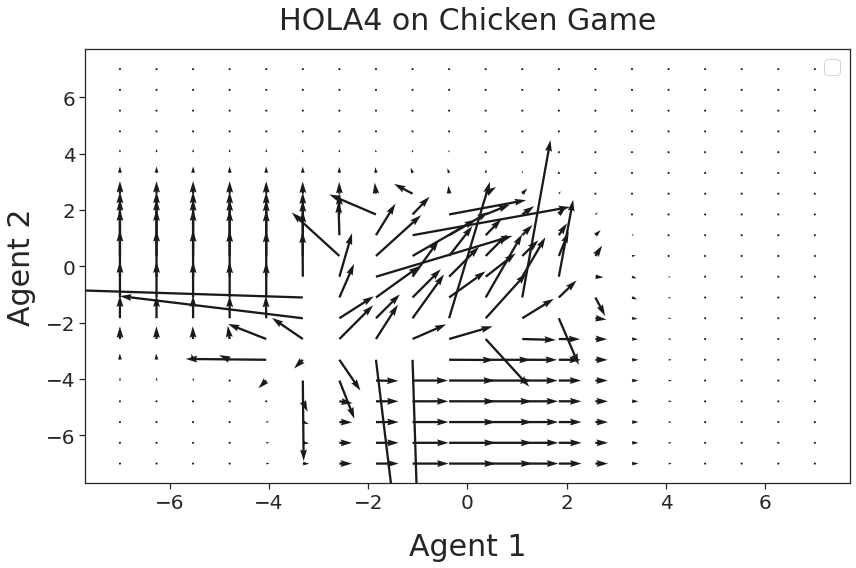}
  \caption{}
  \label{fig:chicken_grad_field4}
\end{subfigure}
\caption{Gradients field of the Chicken game for COLA and HOLA4 at two different look-ahead rates, 0.01 (LHS) and 1.0. (RHS)}
\label{fig:chicken_grad_fields}
\end{figure}

\begin{figure*}[hbt!]
 \centering
 \begin{subfigure}[]{0.49\linewidth}
	\includegraphics[width=\linewidth]{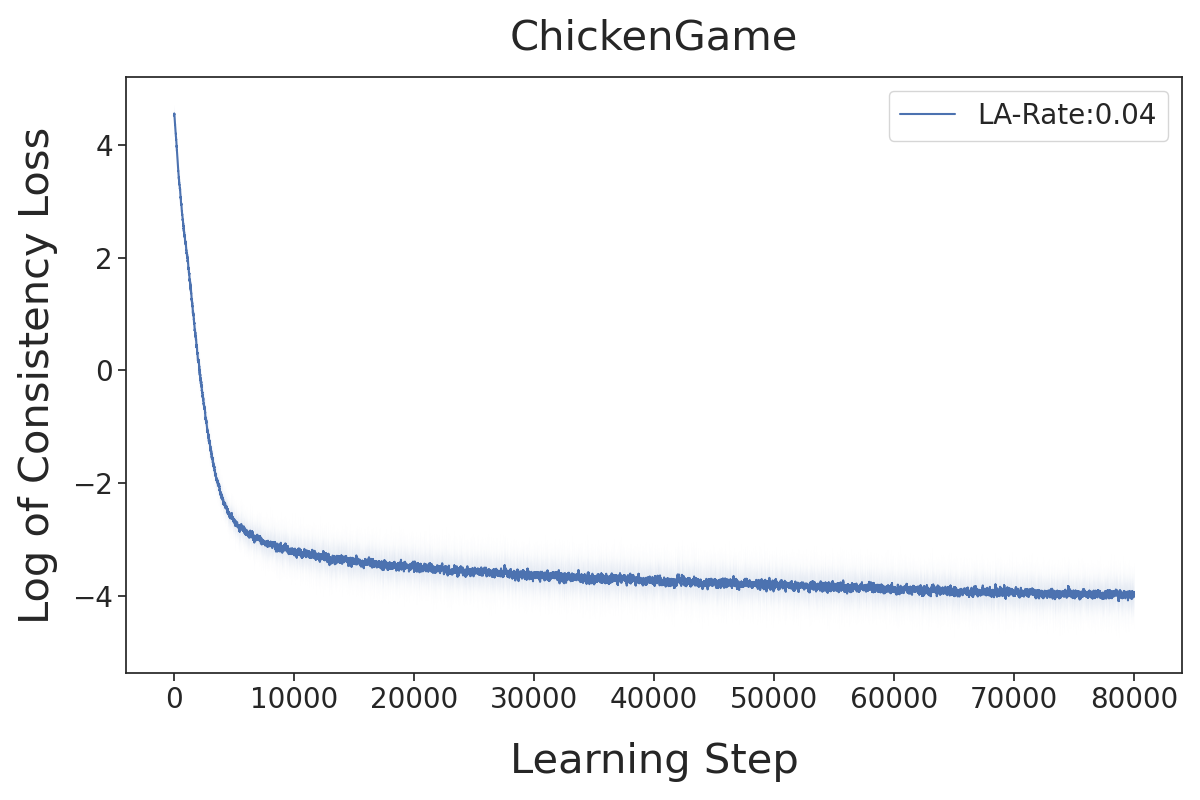}
	\caption{}
	\label{fig:chicken_cons_low}
 \end{subfigure}
 \begin{subfigure}[]{0.49\linewidth}
	\includegraphics[width=\linewidth]{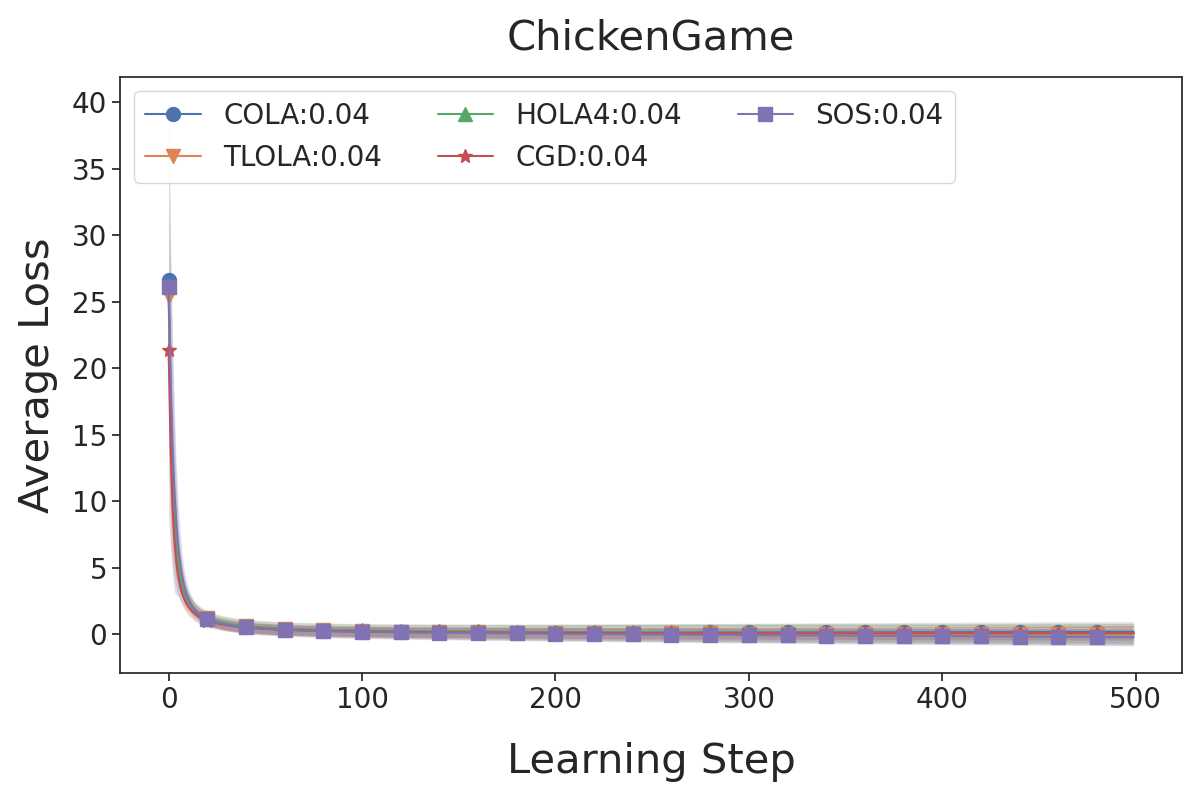}
	\caption{}
	\label{fig:chicken_play_low}
 \end{subfigure}
\begin{subfigure}[]{0.49\linewidth}
  	\includegraphics[width=\linewidth]{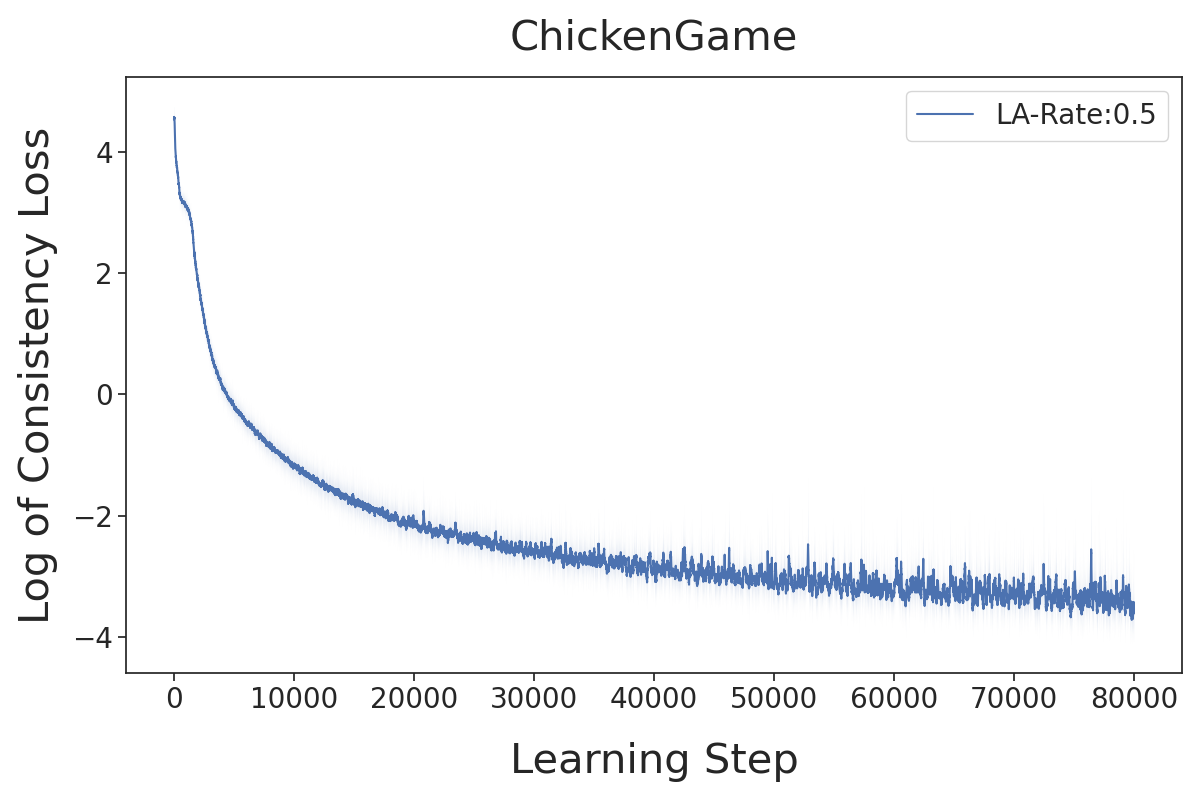}
  	  	\caption{}
  	\label{fig:chicken_cons_high}
 \end{subfigure}
\begin{subfigure}[]{0.49\linewidth}
  	\includegraphics[width=\linewidth]{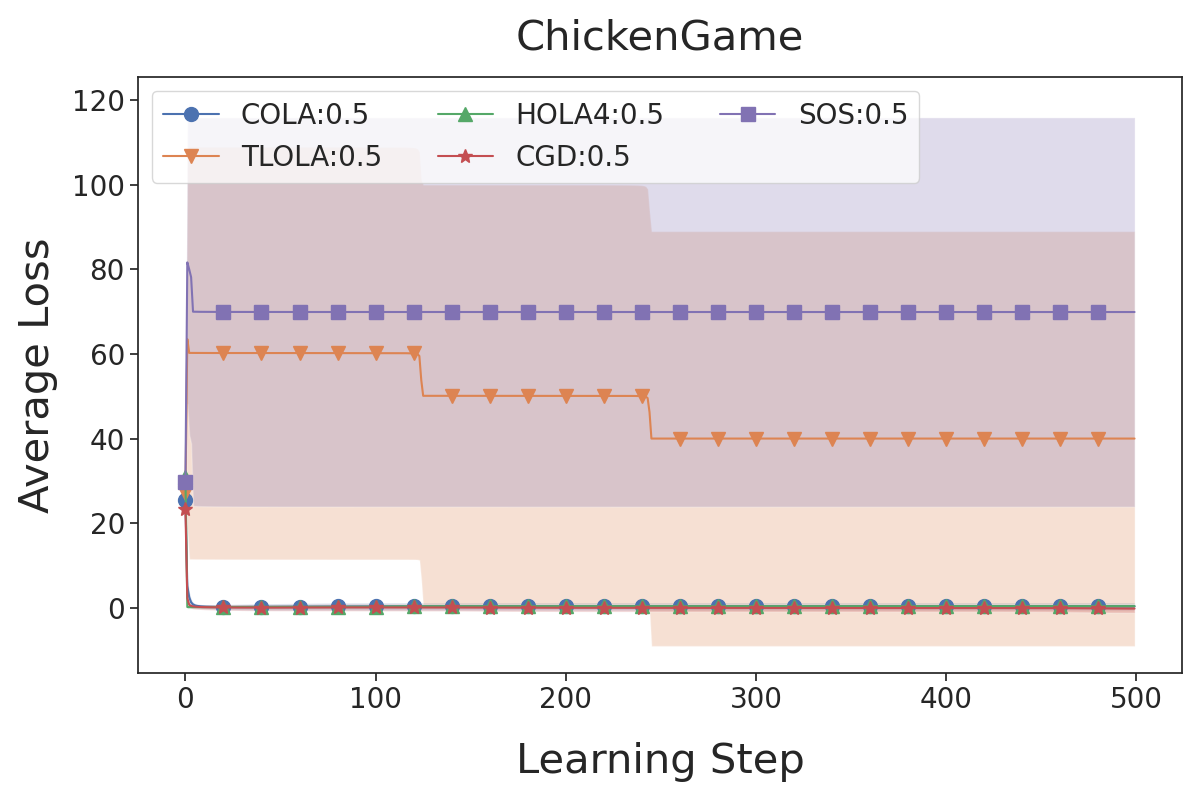}
  	  	\caption{}
  	\label{fig:chicken_play_high}
 \end{subfigure}
 \caption{(a) and (c): Consistency losses of COLA at different look-ahead rates. (b) and (d): Solutions on the Ultimatum game by COLA, TLOLA, HOLA4, CGD, and SOS with a look-ahead rate of 0.04 and 0.5. We used a standard deviation of 1.0 to initialize the parameters.}
 \label{fig:chicken_full}
\end{figure*}

\newpage

\subsection{IPD}
\label{appendix:ipd}
As a baseline comparison, we show the performance of different state-of-the-art algorithms in Figure \ref{fig:cgd_on_ipd}. CGD does not recover the tit-for-tat policy whereas exact LOLA, Taylor LOLA and SOS do.
\begin{table}[hbt!]
\centering
\caption{Payoff Matrix for the IPD game.}
\begin{tabular}{l|l|l}\label{tab:ipd}
     & C     & D     \\ \hline
C & (-1, -1) & (0, -3) \\ \hline
D & (0, -3) & (-2, -2) \\ \hline
\end{tabular}
\end{table}

\begin{figure}[hbt!]
  \centering
  \includegraphics[width=.99\linewidth]{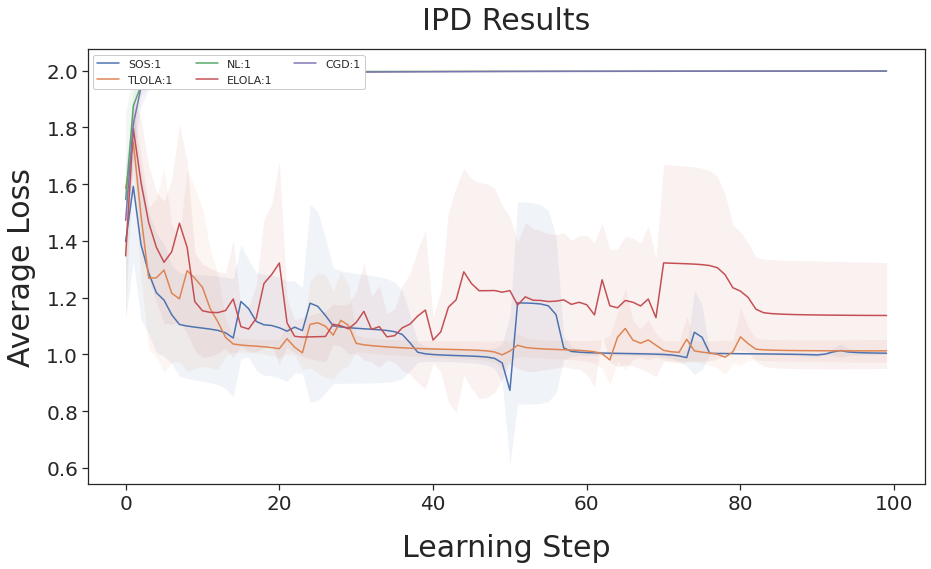}
\caption{CGD, SOS, Taylor LOLA (TLOLA), Exact LOLA (ELOLA) and Naive Learning (NL) on the IPD at a look-ahead rate of 1.0. We used a standard deviation of 1.0 to initialize the parameters.}
\label{fig:cgd_on_ipd}
\end{figure}

\end{document}